\documentclass[twoside,11pt]{article}

\usepackage[preprint]{jstyle}

\usepackage{pgf}
\usepackage{hyperref}
\usepackage{url}
\usepackage{graphicx, subcaption}
\usepackage{amsfonts,amsmath,mathtools}
\usepackage{xcolor}
\usepackage{cleveref}
\usepackage{tikz}
\usepackage{tikz-3dplot}
\usepackage{booktabs}
\usepackage{xspace}
\usepackage[rightcaption]{sidecap}
\usepackage{siunitx}
\usepackage{amsmath,amsfonts,bm}

\def\eqref#1{equation~\ref{#1}}

\def\1{\bm{1}}

\def\rd{{\textnormal{d}}}

\DeclareMathAlphabet{\mathsfit}{\encodingdefault}{\sfdefault}{m}{sl}
\SetMathAlphabet{\mathsfit}{bold}{\encodingdefault}{\sfdefault}{bx}{n}

\def\sT{{\mathbb{T}}}

\newcommand{\E}{\mathbb{E}}

\newcommand{\hs}{\hat{s}}

\usetikzlibrary{positioning}
\usetikzlibrary{calc}
\usetikzlibrary{decorations}
\usetikzlibrary{decorations.text}
\usetikzlibrary{decorations.pathreplacing}

\definecolor{violet}{RGB}{87,6,140}
\definecolor{ultraviolet}{RGB}{137,0,225}
\definecolor{deepviolet}{RGB}{51,6,98}
\definecolor{lightviolet}{RGB}{171,130,197}

\definecolor{darkgrey}{RGB}{64,64,64}
\definecolor{mediumgrey}{RGB}{184,184,184}
\definecolor{lightgrey}{RGB}{242,242,242}

\definecolor{blue}{RGB}{62,178,212}
\definecolor{teal}{RGB}{0,156,139}

\definecolor{dark-gray}{gray}{0.3}
\definecolor{dkgray}{rgb}{.4,.4,.4}
\definecolor{dkblue}{rgb}{0,0,.5}
\definecolor{dkgreen}{rgb}{0,0.5,.0}
\definecolor{rust}{rgb}{0.5,0.1,0.1}

\newcommand{\xdomain}{\mathcal X}

\newcommand{\bR}{\mathbb{R}}

\newcommand{\bE}{\mathbb{E}}

\newcommand{\id}{\mathrm{id}}

\newcommand{\Qcal}{\mathcal{Q}}

\newcommand{\Loss}{\mathrm{L}}

\newcommand{\LD}{\mathrm L_{\mathrm{DICE}}}
\newcommand{\hLD}{\hat{\mathrm{L}}_{\mathrm{DICE}}}
\newcommand{\LAM}{\Loss_{\mathrm{AM}}}

\newcommand{\cAM}{\Loss_{\mathrm{AM}}^{\mathrm c}}

\newtheorem{theorem}{Theorem}
\newtheorem{lemma}[theorem]{Lemma} 
\newtheorem{proposition}[theorem]{Proposition} 
\newtheorem{remark}[theorem]{Remark}
\newtheorem{corollary}[theorem]{Corollary}

\newtheorem{assumption}{Assumption}

\renewcommand{\eqref}[1]{(\ref{#1})}

\usepackage{lastpage}
\jmlrheading{Volume}{Year}{1-\pageref{LastPage}}{??/??; Revised ??/??}{?/??}{21-0000}{}

\ShortHeadings{DICE}{Blickhan, Berman, Stuart, Peherstorfer}
\firstpageno{1}

\crefname{proposition}{proposition}{propositions}
\Crefname{proposition}{Proposition}{Propositions}
\crefname{theorem}{theorem}{theorems}
\Crefname{theorem}{Theorem}{Theorems}
\crefname{definition}{definition}{definitions}
\Crefname{definition}{Definition}{Definitions}
\crefname{assumption}{assumption}{assumptions}
\Crefname{assumption}{Assumption}{Assumptions}

\begin{document}

\title{DICE: Discrete inverse continuity equation for learning population dynamics}

\author{\name Tobias Blickhan \email tobias.blickhan@nyu.edu \\
       \addr Courant Institute of Mathematical Sciences, New York University
       \AND
       \name Jules Berman \email jmb1174@nyu.edu \\
       \addr Courant Institute of Mathematical Sciences, New York University
       \AND
       \name Andrew Stuart \email astuart@caltech.edu \\
       \addr Computing and Mathematical Sciences, California Institute of Technology
       \AND
       \name Benjamin Peherstorfer \email pehersto@cims.nyu.edu \\
       \addr Courant Institute of Mathematical Sciences, New York University}

\maketitle

\begin{abstract}
We introduce the Discrete Inverse Continuity Equation (DICE) method, a generative modeling approach that learns the evolution of a stochastic process from given sample populations at a finite number of time points. Models learned with DICE capture the typically smooth and well-behaved population dynamics, rather than the dynamics of individual sample trajectories that can exhibit complex or even chaotic behavior. 
The DICE loss function is developed specifically to be invariant, even in discrete time, to spatially constant but time-varying spurious constants that can emerge during training; this invariance increases training stability and robustness.
Generating a trajectory of sample populations with DICE is fast because samples evolve directly in the time interval over which the stochastic process is formulated, in contrast to approaches that condition on time and then require multiple sampling steps per time step.
DICE is stable to train, in situations where other methods for learning population dynamics fail, and DICE generates representative samples with orders of magnitude lower costs than methods that have to condition on time. Numerical experiments on a wide range of problems
from random waves, Vlasov-Poisson instabilities and high-dimensional chaos are included to justify these assertions.
\end{abstract}


\tableofcontents

\section{Introduction}

\subsection{Learning sample versus population dynamics for generative modeling}
Learning models of time-dependent stochastic processes $X(t)$ to generate more samples is a key challenge in machine learning and in the computational sciences. We distinguish between learning sample and population dynamics: learning the sample dynamics means finding a model that reproduces the sample trajectories of $X(t)$. In contrast, learning population dynamics means finding a model of the dynamics of the law $\rho(t)$ of $X(t)$.

Learning population instead of sample dynamics can be beneficial in multiple respects. For example, consider the Brownian motion $\rd X(t) = \rd W_t$, for which learning a model of sample dynamics would need to re-produce the erratic and non-differentiable trajectory of a random walk. In contrast, the population dynamics are given by $\partial_t \rho = \frac{1}{2} \Delta \rho$ so that samples with the same population dynamics can be generated with $\rd X(t) = - \frac{1}{2} \nabla \log \rho(t, X(t))$, which lead to smoother and well-behaved trajectories. It is important to note that even though the generated samples agree on the population level with the samples from the random walk, they differ starkly on a sample trajectory level. Because population dynamics can be smooth and, in some sense, simple even when sample dynamics are complex, it means that matching populations has to ignore certain information intrinsic to individual sample trajectories.  However, we argue that in many cases this simplification is acceptable and effectively serves as a form of reduced modeling \citep{Rozza2008,doi:10.1137/130932715,annurev:/content/journals/10.1146/annurev-fluid-121021-025220}. In particular, population-level quantities of interest that depend only on the sample population as a whole, rather than on individual trajectories, can be accurately captured by the inferred population dynamics.

Let us motivate learning population dynamics over sample dynamics by two more examples. First, consider an incompressible fluid with constant density in space
and time. Then Lagrangian trajectories (samples) comprising the fluid can have complicated, even chaotic, dynamics; in contrast, on the population level, the dynamics are constant because the fluid is stationary and its density constant in
space and time. This example shows again that population dynamics can be simple (in this case constant) even when sample dynamics are complex.
Second, consider a chaotic system such as the Lorenz system \citep{DeterministicNonperiodicFlow}, subject to small Brownian noise. Individual sample trajectories are very difficult to predict after a short time,
due to the provably chaotic nature of the noise-free problem \citep{tucker1999lorenz,melbourne2005almost,melbourne2008large}. The population of samples, however, evolve according to a hypoelliptic partial differential equation, under some elementary H\"ormander Lie bracket conditions on the relationship between
the noise and the drift vector fields; these conditions are trivially satisfied
for additive noise in all three components, and also hold beyond this setting \citep{mattingly2002ergodicity}. As a consequence the population
density is smooth in space and in time and, due to ergodicity, converges to a unique stationary solution as time grows; these properties make prediction of the density evolution more straightforward than that of trajectories; empirically similar behavior is observed even in the noise-free case \citep{Chaos}. This example suggests that  learning population dynamics opens a path towards modeling chaotic systems from data in a way that can support scientific and engineering applications. 

\subsection{Literature review}
There is a large body of work on learning population dynamics; see \cite{lavenant_towards_2024} for a survey.

\paragraph{Interpolation on Wasserstein space} 
There are several approaches that construct a piecewise geodesic path through the space of probability measures equipped with the Wasser\-stein-2 metric 
to approximate the curve $t \mapsto \rho(t)$. 
\citet{tong_trajectory_2020} write the problem in dynamic form and then approximate the transported vector field with a neural network. Evaluating the transport cost, i.e. the loss function for training the neural network, requires the integration of the dynamics as in continuous normalizing flow 
methods and can be more costly compared to a simulation-free approach as our approach.
There are other methods that encode the geodesics between time marginals by approximating the vector field that generates them using neural networks  
\citep{bunne_supervised_2023}. There are also works that use splines instead of piecewise linear approximations in Wasserstein space \citep{benamou_second-order_2019, chewi_fast_2021}. 
The main difference between these methods and our approach is that we avoid having to solve the non-linear optimal transport problem and instead only solve a sequence of linearized problems. We note though that this means that we critically build on the assumption that marginal information is available sufficiently densely in time (c.f. \Cref{prop:DICE:BoundForAllT}).

There is a line of work that builds on Schrödinger bridge matching so that the inferred trajectories do not have to follow optimal transport paths  
\citep{chen_deep_2023, leonard_survey_2014,hong_trajectory_2025}. 
However, these methods describe global optimization problems connecting all marginals together and as a result can be costly to train.   

Closest to our approach is Action Matching (AM) introduced by \cite{neklyudov_action_2023}. The AM loss is derived in continuous time and subsequentially discretized. This approach can introduce training instabilities due to spurious constants that emerge during training \citep{berman2024parametric}. We will show  that our loss avoids these training instabilities.

\paragraph{Approaches inspired by the JKO scheme} 
Several of the previously mentioned approaches are based on the JKO scheme introduced by \citet{jordan_variational_1998}. Instead of learning the optimal transport maps between two marginals, the JKO formulation assumes that the observed dynamics follow a gradient flow in the Wasserstein space. This means that the vector field generating the dynamics is given by the gradient of a functional, which becomes the target of inference \citep{bunne_proximal_2022}. The JKO scheme describes a bi-level optimization problem and thus can be challenging to solve. Note that  JKO schemes have also been used as surrogate models to solve gradient flow equations when the form of the energy potential is known \citep{lee_deep_2024}. Another approach inspired by JKO has been introduced by \cite{terpin_learning_2024}. The authors compute exactly the optimal coupling between adjacent empirical marginals using a linear program. 
Once the couplings have been computed, finding the vector field that generates them is a regression problem akin to inferring the dynamics from  sample trajectory data. The authors show that this method 
compares favorably in terms of runtime to other JKO schemes. However, the approach still has to compute the couplings between the time marginals, which we avoid.

\paragraph{Conditional generative models} 
Flow-based and diffusion models \citep{song2020generativealg,Onken_Wu,rombach_high-resolution_2022,lipman2022flow, albergo_stochastic_2023,liu_flow_2023} are established methods to learn the transport from one distribution to another. When these methods are conditioned (i.e.\ made dependent) on time $t$, it is possible to generate samples for all time marginals $\rho(t)$, starting from some reference distribution. However, this requires performing a separate inference step from the reference to the target $\rho(t_j)$ for each time step $t_1, \dots, t_K$ of interest. In contrast, our proposed method goes directly from $\rho(t_j)$ to $\rho(t_{j+1})$ in one step, which can speed up the generation of a sample trajectory by orders of magnitude, as our experiments demonstrate. 

\paragraph{Dynamic approaches for Bayesian inference}
There are transport-based approaches for Bayesian inference, which are related to what we refer to as population dynamics  \citep{reich_nonparametric_2013,myers_sequential_2021, ruchi_fast_2021,pmlr-v235-tian24c}. Often these are formulated as sequential transport maps between the prior and posterior distribution. As a result, they also tackle an interpolation problem on Wasserstein space with marginal constraints, but usually do so using kernel methods. Closest to our approach is the work by \citet{maurais_sampling_2024}, in which the curve $t \mapsto \rho(t)$ is approximated with a vector field that is given 
via a system of coupled, time-parametrized Poisson equations. However, \citet{maurais_sampling_2024} have access to the density, which is common in Bayesian inverse problem settings but unavailable in our data-driven setting. 

\paragraph{Optimal transport for reduced modeling} There is an increasing interest \citep{ehrlacher_nonlinear_2020, iollo_mapping_2022, blickhan_registration_2024, doi:10.1137/20M1316998,khamlich_optimal_2025} in leveraging techniques from optimal transport theory to derive reduced models of problems with moving features and advection effects, which are limited by the Kolmogorov barrier \citep{P22AMS}. However, these methods build on knowledge of the governing equations of the underlying phenomenon, which is unavailable in our setting where we have only data in the form of samples. Our approach can be interpreted as a data-driven, non-intrusive model reduction approach that is applicable to particle problems and other problems dominated by transport. 

\subsection{Our approach: Discrete inverse continuity equation (DICE)}
\paragraph{Inverting the continuity equation}
The starting point of our approach is the continuity equation $\partial_t \rho = - \nabla \cdot (\rho u)$ that describes the dynamics of the law $\rho$ over time. The key quantity in the continuity equation is the vector field $u$ that determines the dynamics. Given data in form of independent samples of the process $X(t)$, our goal is learning $u$.  

We develop a loss function for learning $u$ by constructing the time-discrete weak form of the continuity equation and then deriving an objective function that has the time-discrete weak form as Euler-Langrange equations. We show that the proposed DICE loss function admits a unique minimizer and prove a bound on the error of the learned vector field $\hat{u}$ compared to the actual field $u$. 

Learning a vector field $\hat{u}$ with the DICE loss is simulation free so that there is no need to differentiate through time integration schemes, which avoids the corresponding high computationally costs. Evaluating the DICE loss function requires only empirically estimating expectations over the law $\rho(t)$ using the available training data samples. While our approach builds on the geometry induced by the Wasserstein metric on probability space, we do not compute optimal transport distances between sample populations, which also avoids potentially expensive computations. Once we have learned $\hat{u}$, we can use it to generate samples from different initial conditions using standard, off-the-shelf samplers; we also discuss extensions to generalize over different physics parameters. In one inference step, we generate a whole sample trajectory, which is in stark contrast to standard generative modeling approaches that need to condition on time and perform, for every physical time step, a sequence of sampling steps along an artificial inference (or denoising) time.

\paragraph{Maintaining structure in empirical loss for stable training} 
A key aspect of the DICE loss is that it was designed with training stability in mind. Even more, the focus is on maintaining training stability with the empirical loss, which in our case of time-dependent dynamics means after discretization in time, rather than only with the population loss. While there is a time-continuous counterpart, the DICE loss itself is derived directly from the time-discrete weak form of the continuity equation. This means that the proper time discretization of the continuity equation is inherited by the DICE loss. 
We show that because of the judicious time discretization, the DICE loss satisfies a key property, namely the invariance to additions with spatially-constant-but-time-varying functions even in discrete time. The invariance property of the DICE loss is essential for stable training because otherwise spurious constants during the training can lead to instabilities. In particular, naively discretizing in time can lead to discrete (empirical) loss functions that are ill-defined, which means they are not bounded from below when using nonlinear and common parametrizations such as neural networks. In contrast, we show that the DICE loss defines a well-posed discrete optimization problem and leads to a more stable training behavior than other loss functions for learning population dynamics such as the AM loss  \citep{neklyudov_action_2023}.   We believe this demonstrates the importance of developing not only population loss functions formulated over continuous time but to also take the next step and derive proper empirical loss functions in discrete time.  

We show on several examples from physical science applications that DICE can accurately infer the population dynamics from independent samples alone. The well-posed training objective of DICE is a crucial foundation for any further optimization such as hyper-parameter tuning and the choice of parametrizations (i.e.\ neural network architectures). We provide a thorough presentation of the former to allow future work on the latter.

\subsection{Summary of contributions}
\begin{itemize}
\item Introducing the DICE loss function, which is derived from the time-discrete weak form of the continuity equation to infer the vector fields driving population dynamics. The DICE loss avoids the need to compute expensive optimal transport distances or differentiate through time integration schemes.  
\item Identifying key invariances of learning population dynamics and preserving them in the DICE loss in order to guarantee well-posed optimization problems 
and to enhance training stability.   
\item Deriving a meaningful infinite-sample, continuous-time limit of the DICE loss and establishing a connection to elliptic partial differential equations, which allows leveraging concepts from elliptic theory to analyze properties of the DICE loss. 
\item Achieving fast inference of trajectories of sample populations by evolving samples over the time over which the stochastic process is formulated, in contrast to other methods that condition on time and require multiple sampling steps per time step.
\item Demonstrating, in the context of problems in which the population dynamics are relatively simple and predictable even though the sample dynamics are complex and chaotic, that focusing on population dynamics can be beneficial in reduced modeling.  
\end{itemize}

\subsection{Outline}
In \Cref{sec:PreliminariesAndProblemFormulation} we set the stage by discussing that transport phenomena can be described either via sample or population dynamics. We discuss a continuous formulation of learning vector fields from time marginals in \Cref{sec:Cont}, which leads to the time-discrete loss function in \Cref{sec:DICE}. \Cref{sec:DICEInfinite} shows that if we take the limit of having the time marginals available at all times $t$ instead of only at a finite number of discrete time points, we recover the time-continuous formulation of \Cref{sec:Cont}. A comparison to AM is provided in \Cref{sec:CmpAM}, which discusses the training instabilities that are avoided by DICE. Practical aspects of learning vector fields with the DICE loss are discussed in \Cref{sec:DICESamples}. Numerical experiments with various science applications in \Cref{sec:NumExp} provide further evidence that training with the DICE loss is stabler than with other loss functions. Conclusions are drawn in \Cref{sec:Conc}.

\section{Describing transport phenomena}
\label{sec:PreliminariesAndProblemFormulation}
In this section, we recap results on describing transport phenomena. In \Cref{sec:T:Vectorfields}, we discuss the formulation of transport phenomena via the continuity equation, which corresponds to a population or Eulerian perspective. In \Cref{sec:T:GenSamples}, we explore the formulation via flows of vector fields, which corresponds to a sample or Lagrangian perspective. 
    
\subsection{Vector fields of flows}\label{sec:T:Vectorfields}
 Throughout this work, $\mathcal{X}$ denotes a bounded, connected, and open subset of $\mathbb{R}^d$, or the flat torus $\mathbb{T}^d$, and $\mathcal{P}(\mathcal{X})$ denotes the space of probability measures on $\mathcal{X}$ with finite second moment. Consider a stochastic process $X(t) \in \mathcal{X}$ with time $t \geq 0$. The law of $X(t)$ is $\rho(t) \in \mathcal{P}(\mathcal{X})$. When we wish to view $\rho(t)$ as a density over $\mathcal{X}$ we will write $\rho(t, \cdot): \mathcal{X} \to \mathbb{R}$ and denote its value as $\rho(t, x)$ at $x$. 
The map $t \mapsto \rho(t)$ is a curve through $\mathcal{P}(\mathcal X)$. The value of the curve at time $t$ is $\rho(t)$. 
   
We encode the curve $t \mapsto \rho(t)$ through a vector field $u: [0, T] \times \xdomain  \to \bR^d$ in the following sense: we say that the pair $(u, \rho)$ solves the continuity equation in the weak sense if for all smooth and compactly supported test functions $\mathcal C_0^\infty(\mathcal X) \ni \varphi: \mathcal X  
\to \bR$ and for almost every $t \in [0,T]$ the map $t \mapsto \bE_{x\sim\rho(t)}[\varphi(x)]$ is absolutely continuous and
\begin{align}
    \label{eq:Prelim:ContWeakForm}
    \frac{\rd}{\rd t} \bE_{x\sim\rho(t)}[\varphi(x)] = \bE_{x\sim\rho(t)}[u(t, x) \cdot \nabla \varphi(x)]
\end{align} holds. Recall that the functions in $\mathcal{C}^{\infty}_0(\mathcal{X})$ evaluate to zero at the boundary of the domain $\mathcal{X}$. We assume that the probability density is at all times contained in $\mathcal{X}$, hence $u$ is tangential to the boundary of $\mathcal X$ at all times. 
In the periodic setting, i.e., when $\mathcal{X} = \sT^d$ is the flat torus, then $\mathcal{C}^{\infty}_0(\sT^d) = \mathcal{C}^{\infty}(\sT^d)$ and periodic boundary conditions are imposed. 
Following \cite{gigli_second_2011}, we call a pair $(u, \rho)$ solving the continuity equation a transport couple. Additionally, we call $u$ compatible with $\rho$ if $(u, \rho)$ is a transport couple. 

   Every weak solution $(u, \rho)$ of the continuity equation over $[0,T]$ is also a distributional solution between $\rho(0)$ and $\rho(T)$ in the sense that 
    \begin{align}\label{eq:Prelim:ContVarForm}
        \mathbb E_{x \sim \rho(T)} \left[ \varphi(T, x) \right] - \mathbb E_{x \sim \rho(0)} \left[ \varphi(0, x) \right] 
        = \int_0^T \mathbb E_{x \sim \rho(t)} \left[ (\partial_t \varphi + u \cdot \nabla \varphi)(t, x) \right] \rd t
    \end{align}
    holds for all smooth compactly supported functions $\mathcal C_0^\infty( [0,T] \times \mathcal{X}) \ni \varphi: [0,T] \times \mathcal{X} \to \bR$. 
    The converse is also true; see \citet[Lemma 8.1.2]{ambrosio_gradient_2005} and  \citet[Proposition 4.2]{santambrogio_optimal_2015}.

\subsection{Generating samples with vector fields}\label{sec:T:GenSamples}
To keep the following results concrete, we make the following assumption on $\rho$.

    \begin{assumption}[Poincaré inequality]
    \label{assumption:prelim:poincare}
        For all $t\in [0,T]$, $\rho(t)$ admits a density function $\rho(t, \cdot): \mathcal{X} \to \mathbb{R}$ that is absolutely continuous with respect to the Lebesgue measure on $\mathcal X$. Furthermore, the density function satisfies a Poincaré inequality with constant $\lambda(t) > 0$:
        \begin{align}\label{eq:DICE:PoincareInEq}
            \left\| g - \mathbb E_{x \sim \rho(t)}[g] \right\|^2_{L^2(\rho(t))} \leq \frac{1}{\lambda(t)} \left\| \nabla g \right\|^2_{L^2(\rho(t))}
        \end{align}
        for any Lipschitz function $g: \mathcal{X} \to \mathbb{R}$ \cite[Definition 21.17]{villani_optimal_2009}.
\end{assumption}
Because of the absolute continuity of the densities assumed in  \Cref{assumption:prelim:poincare},  a function $x \mapsto g(x)$ that is Lipschitz is also differentiable $\rho(t)$-a.e.\ in $\mathcal{X}$ by Rademacher's theorem. If the domain $\mathcal X$ allows for a Poincaré constant $\lambda_{\mathcal X}$ with respect to the Lebesgue measure, then Poincar\'e inequality \eqref{eq:DICE:PoincareInEq} holds for densities that are bounded below by $\underline{\rho}(t) > 0$, with $\lambda(t) \geq \underline{\rho}(t) \lambda_{\mathcal X}$.

If $\mathcal X = \bR^d$, then measures of the form $\rho(t, x) = e^{-V(t,x)}$ satisfy a Poincaré inequality as long as $V$ satisfies a linear growth condition \citep[Corollary 1.6]{bakry_simple_2008}. A slightly stronger result captures the case of Gaussian measures: When the Hessian of $V(t, \cdot)$ satisfies $D^2_x V(t, \cdot) \succ \alpha \, \mathrm{Id}$, then $\lambda(t) \geq \alpha$. The bibliographical notes in \citet[Chapter 21]{villani_optimal_2009} give an overview over these arguments.
    
If $\rho(t)$ admits a sufficiently regular density function  $\rho: [0, T] \times \mathcal{X} \to \mathbb{R}$, then we can consider the strong form of the continuity equation
    \begin{equation}
    \label{eq:prelim:continuity}
        \partial_t \rho(t, x) + \nabla \cdot \left( \rho(t, x) u(t, x) \right) = 0\,, \quad \forall (t, x) \in [0, T] \times \xdomain\,\,.
    \end{equation}
Furthermore, under assumptions outlined in \citet[Proposition 8.1.8]{ambrosio_gradient_2005}, the flow $\phi_t: \mathcal{X} \to \mathcal{X}$ induced by $u$ and the corresponding ordinary differential equation (ODE),
\begin{align}\label{eq:Prelim:ODE}
        \frac{\rd}{\rd t} \phi_t(a) = u(t, \phi_t(a)), \qquad \phi_0(a) = a \,,
    \end{align}
    defines $\phi_t(a) = \widehat X(t) \sim \rho(t)$ for an initial $a \sim \rho(0)$. In particular, this allows us to use the field $u$ to generate sample trajectories $\widehat X(t)$ that follow the law $\rho(t)$ with
    \begin{equation}\label{eq:Prelim:ODESampling}
    \frac{\rd}{\rd t}\widehat X(t) = u(t, \widehat X(t))\,, \qquad \widehat X(0) \sim \rho(0)\,.
    \end{equation}

\section{Learning vector fields from time marginals}\label{sec:Cont}
In this section, we formulate an optimization problem for inferring advecting vector fields from density time marginals. Critically, we allow only to integrate against time marginals; we do not assume we can evaluate the corresponding densities. In \Cref{sec:TCont:MinEnergy}, we discuss that we will focus on inferring vector fields with minimal kinetic energy and highlight the fact that such vector fields have to be gradient fields. \Cref{sec:ContT:InfinitTransport,sec:ContT:Elliptic} show that inference within the class of  gradient fields allows the optimization problem to be related to infinitesimal optimal transport and a variational formulation of elliptic partial differential equations (PDEs), respectively. In \Cref{sec:ContT:ReWriting}, we re-write the variational formulation of the elliptic problem to bring it into a form that motivates the DICE loss function in the next section. 

\subsection {Minimal energy and tangent fields}\label{sec:TCont:MinEnergy}
Given the time marginals $\rho(t)$ with $t \in [0, T]$, in the sense that we can integrate against them, the task at hand is to invert the continuity equation \eqref{eq:Prelim:ContWeakForm} to find a vector field ${u}$ that is compatible with  $\rho$. We stress that we aim to learn a vector field for one path of marginals; learning vector fields that generalize over different initial laws at time $t = 0$ or even different time marginals can be achieved to some extent via the parameter-dependent formulation that we present in  Section~\ref{sec:DICE:ParamCase}. 

The continuity equation given in \eqref{eq:Prelim:ContWeakForm} (and its strong form given in \eqref{eq:prelim:continuity}) does not uniquely determine a vector field given time marginals. To see this, consider a density function with $\rho(t, x) > 0$ and let $(\rho, u)$ be a transport couple. Then, the pair $(\rho, u + \rho^{-1}{w})$ is also a transport couple for any sufficiently regular vector field $w: \xdomain \rightarrow \bR^d$ that is divergence-free $\nabla \cdot w = 0$.  
Because the continuity equation does not uniquely determine a vector field, we impose the additional condition that the vector field has to have minimal kinetic energy over time, 
\begin{align}\label{eq:Prelim:KinEnergy}
    J(u) = \frac{1}{2} \int_0^T\mathbb{E}_{x \sim \rho(t)}\left[|u(t, x)|^2\right]\mathrm dt\,.
\end{align}
It is shown in \citet[Proposition 8.4.3]{ambrosio_gradient_2005} and \citet[Proposition 1.30]{gigli_second_2011} that the vector field $u$ that minimizes
the kinetic energy $J$ has to be a gradient field: there exists 
a potential field $s: [0,T] \times \mathcal{X} \to \mathbb{R}$ with $u = \nabla s$. The gradient is to be understood as
\begin{align}\label{eq:Prelim:TangentSpace}
        \nabla s(t, \cdot) \in \overline{ \{ \nabla \phi: \phi \in C^\infty(\xdomain) \} }^{L^2(\rho(t))}\,, \quad  t \in [0, T]\,,
    \end{align}
which is the closure of the set of gradients $\{ \nabla \phi: \phi \in C^\infty(\xdomain) \}$ with respect to $L^2(\rho(t))$. Following the work by \cite{otto_geometry_2001}, the space in \eqref{eq:Prelim:TangentSpace} is referred to as the tangent space of $P(\mathcal X)$ at $\rho(t)$ in \citet[Definition 8.4.1]{ambrosio_gradient_2005}. The transport couple $(\rho, \nabla s)$ with a gradient field as vector field is unique, which we denote as $\nabla s^* = u^*$. (Note that it is $\nabla s$ that is unique and not $s$ itself.) 
    
The fact that minimizing the kinetic energy implies gradient structure can also be seen with a formal argument under the assumption that $\rho$ admits a density at all times and that the density is strictly positive. To this end, let $u^*$ be a minimizer of $J$ that also satisfies the continuity equation \eqref{eq:prelim:continuity}. It then holds that 
    \begin{align}
    \label{eq:prelim:gradientstructureargument1}
        \frac{\rd}{\rd \varepsilon} \bigg |_{\varepsilon = 0} J(u^* + \varepsilon \rho^{-1}{w}) = 0
    \end{align}
    for any $w$ with $\nabla \cdot w = 0$. Note that $u^* + \varepsilon \rho^{-1}{w}$ also satisfies the continuity equation \eqref{eq:prelim:continuity} for $\rho$, hence it is a valid competitor in the minimization.
    Expanding \eqref{eq:prelim:gradientstructureargument1} leads to 
    \begin{align}
        \int_0^T \int_{\xdomain}  u^*(t, x) \cdot w(t, x) \, \rd x \, \rd t = 0 \quad \forall w: \nabla \cdot w = 0.
    \end{align}
If we consider a vector field $w$ of the form $w(t, x) = \delta_{\tau}(t) \eta(x)$ with $\tau$ arbitrary and $\eta$ divergence free, this implies that $u^*(\tau, \cdot)$ is $L^2$-orthogonal to divergence free vector fields for all $\tau$. By the Helmholtz decomposition, this implies that $u^*(\tau, \cdot)$ must be a gradient field. 
    
In summary, because the transport couple $(\rho, \nabla s^*)$ is unique, when we seek a vector field $u$ that minimizes the kinetic energy \eqref{eq:Prelim:KinEnergy}, it is sufficient to restrict attention to functions $s: [0, T] \times \mathcal{X}  \to \mathbb{R}$ so that  $u = \nabla s$ satisfies the continuity equation. 

\begin{remark}
    Notice that the potential $s$ depends on time. In general, when $u$ is not a gradient field itself, then even if $u$ is not explicitly dependent on time, $\nabla s$ will be. Likewise, when $u$ depends on time through $\rho(t)$ only (as is the case in mean field models), then this dependence can be subsumed by the time dependence of $s$.
\end{remark}

\subsection{Relation to infinitesimal optimal transport}\label{sec:ContT:InfinitTransport}
For arbitrary but fixed time $t$ and time-step size $\Delta t \in \bR$, the minimum energy compatible vector field $\nabla s^*$ is related to the optimal transport from $\rho(t)$ to $\rho(t + \Delta t)$.
From \citet[Proposition 8.4.6]{ambrosio_gradient_2005}, we recall the limit
\begin{align}
\lim_{\Delta t \rightarrow 0} \frac{W_2 \Bigl(\rho(t  + \Delta t), \bigl(\id + \Delta t \nabla s^*(t, \cdot)\bigr)_\# \rho(t)\Bigr)}{|\Delta t|}  = 0
\end{align}
where $W_2$ denotes the Wasserstein-2 metric,  $\id$ the identity map, and $(\id + \Delta t \nabla s^*(t, \cdot))_\# \rho(t)$ the push forward of $\rho(t)$ under the map $\id + \Delta t \nabla s^*(t, \cdot)$. Let now $T_{\rho(t) \rightarrow \rho(t + \Delta t)}$ be the optimal transport map from $\rho(t)$ to $\rho(t + \Delta t)$, then we obtain  
\begin{align}
\label{eq:prelim:OTTangentConstruction}
    \lim_{\Delta t \rightarrow 0} \frac{T_{\rho(t) \rightarrow \rho(t + \Delta t)} - \id}{\Delta t} = \nabla s^* (t, \cdot),
\end{align}
which motivates why the vector field $\nabla s^*$ is called the tangent vector field to the curve $t \mapsto \rho(t) \in (\mathcal P(\xdomain), W_2)$; see Figure~\ref{fig:tangent}.

\begin{figure}
\centering
   \begin{tikzpicture}




\tikzset{shorten >= 3pt}
\coordinate (O) at (-2, 0);   
\coordinate (A) at (0, 0);
\coordinate (B) at (1.5, 2);
\coordinate (C) at (4, 1);
\coordinate (D) at (6, 3);
\coordinate (N) at (3, -2);  

\draw[thick,->] (O) to[out=-20, in=210] (A)
             to[out=40, in=180] (B)
             to[out=0, in=180] (C);

\filldraw (A) circle (2pt);
\filldraw (B) circle (2pt);
\coordinate (approx) at (1.5, 1.2);


\node[above left] at (A) {$\rho(t)$};
\node[above=0pt] at (B) {$\rho(t+\Delta t)$};
\node[above=1pt] at (C) {$t \mapsto \rho(t)$};
\node[violet, right] at (0.6, 0.2) {$(\mathrm{id} + h \nabla s^*(t, \cdot))_\#\rho(t)$};
\node[teal, right] at (1.55, 1.2) {$\mathcal O(\Delta t^2)$};

\draw[violet, thick,->] (A) -- (approx);
\draw[teal,thick,<->] (approx) to[out=30, in=-30] (B);

\end{tikzpicture}
    \caption{Illustration of the relation between $t \mapsto \rho(t)$ as a curve through Wasserstein space and the minimal energy vector field $\nabla s^*$.}
    \label{fig:tangent}
\end{figure}
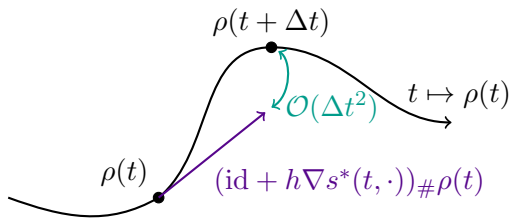

\subsection{An optimization objective via variational formulations of elliptic problems}\label{sec:ContT:Elliptic}
The uniqueness of the pair $(\rho, \nabla s^*)$ established in the previous section means it is sufficient to look for $s$ so that $\nabla s$ satisfies the continuity equation. 
At least formally, we can then plug $\nabla s$ into the strong form \eqref{eq:prelim:continuity} of the continuity equation and obtain the problem
\begin{equation}\label{eq:Cont:EllipticProblem}
-\nabla \cdot (\rho(t, x) \nabla s(t, x)) = \partial_t \rho(t, x)\,.
\end{equation}
If $\partial_t \rho(t) / \rho(t)$ is in $L^{2}(\rho(t))$, the Poincar\'e inequality of Assumption~\ref{assumption:prelim:poincare} holds, and coercivity holds in a $\rho$-weighted sense, then \eqref{eq:Cont:EllipticProblem} is an elliptic PDE. Notice that in the standard, unweighted sense, one would require $\rho(t, x) \geq \underline{\rho} > 0$ to be strictly positive over $[0, T] \times \mathcal{X}$ to obtain coercivity. 
Notice that $\partial_t \rho$ serves as a source term because the unknown in \eqref{eq:Cont:EllipticProblem} is $s$ and not $\rho$. Thus, under the assumptions above that $\rho$ is sufficiently regular, we can interpret finding a vector field $u = \nabla s$ as solving a coupled system of PDE problems \eqref{eq:Cont:EllipticProblem}. Critically, the PDE problems are coupled via the right-hand side $\partial_t \rho$ rather than the unknown $s$. The perspective via PDE problems allows us to leverage the well-known variational formulation of \eqref{eq:Cont:EllipticProblem} to obtain 
\begin{equation}\label{eq:Cont:EllipticVarForm}
\min_s \int_0^T \int_{\mathcal{X}} \frac{1}{2}|\nabla s(t, x)|^2 \rho(t, x) - s(t, x) \partial_t \rho(t, x)\mathrm dx\,\mathrm dt\,,
\end{equation}
which provides an optimization problem for finding an $s$ given $\rho$ and $\partial_t \rho$  \citep{benamou_computational_2000}.

\subsection{Re-writing the time derivative}\label{sec:ContT:ReWriting}

The objective of the variational problem given in \eqref{eq:Cont:EllipticVarForm} depends on the density $\rho$ and its derivative $\partial_t \rho$. 
We will later consider the setting where we have access to samples of the time marginals $\rho(t)$, which can be interpreted as being able to integrate against the density $\rho(t, \cdot)$ over time $t$.  
However, the objective in \eqref{eq:Cont:EllipticVarForm} also depends on the time derivative $\partial_t \rho$ of $\rho$, which is challenging to estimate if one can only integrate against the time marginals $\rho(t)$. 

It is helpful to rewrite the source term $t \mapsto \int_{\mathcal{X}} s(t, x)\partial_t \rho(t, x)\mathrm dx$ in \eqref{eq:Cont:EllipticVarForm} as an expectation with respect to $\rho$. To do so, note that
\begin{multline}\label{eq:Cont:ddtaureform}
\frac{\rd}{\rd t} \mathbb{E}_{x \sim \rho(t)}\left[s(t, x)\right] = \frac{\rd}{\rd \tau} \mathbb{E}_{x \sim \rho(\tau)}\left[s(t, x)\right] \bigg |_{\tau = t} + \frac{\rd}{\rd \tau} \mathbb{E}_{x \sim \rho(t)}\left[s(\tau, x)\right] \bigg |_{\tau = t} \\
= \int_{\mathcal{X}} s(t, x)\partial_t \rho(t, x) \, \mathrm dx + \int_{\mathcal{X}} \rho(t, x) \partial_t s(t, x) \, \mathrm dx
\end{multline}
holds, where $\tau$ is an artificial time. Contained in the identity \eqref{eq:Cont:ddtaureform} is the statement that 
\begin{align}\label{eq:Cont:PartialDiff}
\int_{\mathcal{X}} s(t, x)\partial_t \rho(t, x) \, \mathrm dx = \frac{\rd}{\rd \tau} \mathbb{E}_{x \sim \rho(\tau)}\left[s(t, x)\right] \bigg |_{\tau = t}\,,\qquad t \in (0, T)\,,
\end{align}
which is an expression in which the time marginals $\rho(t)$ enter only via expectations. Using \eqref{eq:Cont:PartialDiff}, we obtain an objective function that only consists of expectations with respect to $\rho(t)$
\begin{equation}\label{eq:Cont:ContLossddtau}
    L(s) = \int_0^T \left( \mathbb E_{x \sim \rho(t)} \left[ \frac{1}{2} |\nabla s(t, x)|^2 \right] - \frac{\rd}{\rd \tau} \mathbb E_{x \sim \rho(\tau)} \left[ s(t, x) \right]\bigg |_{\tau = t} \right) \rd t,
\end{equation}
and which has the same minimizers as the variational formulation \eqref{eq:Cont:EllipticVarForm} over $s$. 
Still, the objective function \eqref{eq:Cont:ContLossddtau} is not a loss function yet because it cannot be directly estimated from samples of the time marginals $\rho(t)$ because of the derivative in the artificial time $\tau$.

\section{Discrete inverse continuity equation (DICE)}
\label{sec:DICE}
In this section, we consider the problem of recovering potential field $s$ that is consistent with samples of the time marginals $\rho(t, \cdot)$ given only at discrete time points $t \in \{t_k\}_{k=0}^K.$ It is a working assumption of our methodology that these points are spaced closely in time, relative to the natural
timescale of the density evolution. \Cref{sec:Prelim:ProblemFormulation} provides a more detailed discussion of the problem setup. In \Cref{sec:DICE:Loss}, we introduce the DICE loss function that has the
potential as a minimizer. The loss is defined through a time-discrete weak form of the continuity equation corresponding to the discrete time points at which time marginals are available. In \Cref{sec:DICE:Properties}, we show that the corresponding optimization problem is well posed and admits a unique minimizer.

\subsection{Learning vector fields of  marginals at discrete time points}
\label{sec:Prelim:ProblemFormulation}
\paragraph{Learning vector fields} We now assume that we have access to $\rho(t_j)$
only at discrete time points $t \in \{t_k\}_{k=0}^K$, where 
access means that we can integrate against $\rho(t_j)$; in particular we
do not assume that we can evaluate the corresponding density function. This requirement is in preparation for later sections when we have available only samples of the time marginals; see Section~\ref{sec:DICESamples}.
We wish to find a vector field $u$ that minimizes the kinetic energy $J(u)$ given in \eqref{eq:Prelim:KinEnergy} and also minimizes a loss function defined by asking
that the continuity equation is compatible with the marginals $\{\rho(t_j)\}_{j = 0}^K$ at $t \in \{t_k\}_{k=0}^K$. This optimization problem is
simplified by recalling that finding a minimal kinetic energy field requires $u = \nabla s$, leading to an optimization problem over $s.$

\paragraph{Piece-wise geodesic interpolants}
One approach to our problems is to connect $\rho(t_{j-1})$ to $\rho(t_{j})$ on
$(t_{j-1} - t_{j})$ with Wasserstein geodesics, for $j = 1, \dots, K$. This requires solving fully nonlinear optimal transport problems between all marginals, which is computationally expensive. Furthermore, the approach does not exploit the
fact that $\Delta t_{\text{max}} := \max_{j = 1, \dots, K} |t_j - t_{j - 1}|$ is
assumed small. Instead, our approach can be interpreted as finding the tangent vector field $\nabla s(t_j, \cdot)$ at the discrete time points $t_j$ for $j = 0, \dots, K$ and approximating the curve $t \mapsto \rho(t)$ via the tangent vector fields $\nabla s(t_0, \cdot), \dots, \nabla s(t_K, \cdot)$ at the discrete time points $t_0, \dots, t_K$; see Figure~\ref{fig:tangent}. 
We will show that the linearization underlying the construction of $\nabla s$ will introduce an error that is also of the order of the maximal time-step size (see \Cref{prop:DICE:BoundForAllT}) and thus our approach does not lead to a worse scaling of the error with the maximal time-step size than connecting Wasserstein geodesics, while solving the simpler linear problem on the tangent space as opposed to the full optimal transport problem.

\subsection{The DICE loss}\label{sec:DICE:Loss}
\paragraph{Function spaces} 
We introduce notation essential for definition of the DICE loss.  We start by defining Hilbert spaces $L^2(\rho(t_j))$
through the $\rho(t_j)-$weighted $L^2$ inner products
    \begin{align}
        \langle s^{1}_j, s^{2}_j\rangle_{L^2(\rho(t_j))} = \int_{\mathcal X} s^{1}_j(x) s^{2}_j(x) \rho(t_j, x) \, \rd x\,, \quad j = 0, \dots, K.
    \end{align}
    We also define the Hilbert spaces
    \begin{align}
        \mathcal S_j = \left\{ s \in L^2(\rho(t_j)): \nabla s \in L^2(\rho(t_j))\right\},\qquad j = 0, \dots, K\,.
    \end{align}
    The spaces $\mathcal{S}_j$ are closely related to Sobolev spaces $H^1(\mathrm dx)$, except that the integration defining $\mathcal{S}_j$ is with respect to the measure $\rho(t_j)$ instead of the Lebesgue measure used to define
    the $H^1(\mathrm dx)$ subspace of $L^2(\mathrm dx)$. Indeed, given upper
    and lower (positive) a.e. bounds on $\rho(t_j, x)$ all the Hilbert spaces 
    $\mathcal S_j$ are equivalent to $H^1(\mathrm dx).$ 
    Likewise the spaces $L^2(\rho(t_j))$ are equivalent to $L^2(\mathrm dx)$.
    The normalized counterpart of $\mathcal{S}_j$ is
    \begin{align}
    \mathcal S_j^0 = \{s \in \mathcal{S}_j \, : \, \mathbb{E}_{x \sim \rho(t_j)}[s(x)] = 0\}\,.
    \end{align}

    The Cartesian product $L^2(\rho(t_0)) \times \dots \times  L^2(\rho(t_K))$ is a Hilbert space, because $L^2(\rho(t_j))$ for $j = 0, \dots, K$ are Hilbert spaces. Likewise, because $\mathcal S_j$ for $j = 0, \dots, K$ are Hilbert spaces, the Cartesian product $\mathcal S_0 \times \dots \times \mathcal S_K$ is a Hilbert space. 
Denote the vector of functions $\bar{s}(x) = [s_0(x), \dots, s_K(x)]$ and define the 
inner product on $L^2(\rho(t_0)) \times \dots \times L^2(\rho(t_K))$ by
    \begin{align}
        \langle \bar{s}^1, \bar{s}^2 \rangle_{L^2(\rho(t_0)\cdots\rho(t_K))} = \sum_{j=0}^K  \int_{\mathcal X} \frac{t_{j+1} - t_{j-1}}{2} s^1_j(x) s^2_j(x) \, \rho(t_j, x) \, \rd x,
    \end{align}
    for $\bar{s}^1, \bar{s}^2 \in L^2(\rho(t_0)) \times \dots \times  L^2(\rho(t_K)).$
  
Furthermore, we note that 
    \begin{align}
        \Vert \bar{s} \Vert_{{L^2(\rho(t_0)\cdots\rho(t_K))}}^2 = \langle \bar{s}, \bar{s} \rangle_{L^2(\rho(t_0)\cdots\rho(t_K))} = \sum_{j=0}^K \frac{t_{j+1} - t_{j-1}}{2} \| s_j(x) \|^2_{L^2(\rho(t_j))}\,,
    \end{align}
defines a norm provided that all $t_0, \dots, t_K$ are distinct; this follows from\
the fact that then
    \begin{align}
        \Vert \bar{s}^1 - \bar{s}^2 \Vert_{L^2(\rho(t_0) \cdots \rho(t_K))}^2 = 0 \Leftrightarrow s_j^1 = s_j^2 \; \rho(t_j)\text{-almost-everywhere}\,, \qquad j = 0, \dots, K\,.
    \end{align}
We also introduce the notation
    \begin{align}
        \Vert \mathrm{D}\bar{s} \Vert_{{L^2(\rho(t_0)\cdots\rho(t_K))}}^2 &:= \sum_{j=0}^K  \int_{\mathcal X} \frac{t_{j+1} - t_{j-1}}{2} | \nabla s_j(x)|^2 \, \rho(t_j, x) \, \rd x,
    \end{align}
    and note that
\begin{align}
    \| \bar s \|_{\mathcal{S}_0 \times \cdots \times \mathcal{S}_K} := \Vert \bar{s} \Vert_{{L^2(\rho(t_0)\cdots\rho(t_K))}}^2+\Vert \mathrm{D}\bar{s} \Vert_{{L^2(\rho(t_0)\cdots\rho(t_K))}}^2
\end{align}
defines a norm, and by polarization an inner product, on $\mathcal S_0 \times \dots \times \mathcal S_K.$

We introduce the set of functions that intersect the Cartesian product $\mathcal{S}_0 \times \cdots \times \mathcal{S}_K$ with extensions over the time interval $[0, T]$ as \begin{align}\label{eq:DICE:SpaceS}
        \mathcal S = \{ s : [0,T] \times \mathcal X  \to \bR: s(t_j, \cdot) \in \mathcal S_j \text{ for } j = 0, \dots, K \}\,,
    \end{align}
    as well as the normalized counterpart $\mathcal{S}^0$ that intersects with $\mathcal{S}_0^0 \times \cdots \times \mathcal{S}_K^0$. Note that we do not make any statement about the behavior of functions $s \in \mathcal S$ outside of the set of time points $\{t_j\}_{j=0}^K$ for now. 

\paragraph{Weak form of continuity equation in discrete time} Let us now turn to the discrete-weak form of the continuity equation. Recall from the paragraph
on ``Learning vector fields'' that we seek a vector field defined as the gradient of potential field  $\hat{s} \in \mathcal{S}$ that satisfies the weak form of the continuity equation \eqref{eq:Prelim:ContWeakForm} at the $K + 1$ discrete time points $t_0 < t_1 < \dots < t_K$ at which the marginals $\{\rho(t_j)\}_{j = 0}^K$ are available. We use finite differences to approximate the time derivative on the left-hand side of \eqref{eq:Prelim:ContWeakForm} and use the convention that $t_{-1} = t_0$ and $t_{K+1} = t_K$, leading to the following equations that we would like to choose 
$\hat{s}$ to satisfy:
\begin{align}
    \label{eq:DICE:DiscreteCompat}
    \frac{\bE_{x\sim\rho(t_{j+1})}[\varphi(x)] - \bE_{x\sim\rho(t_{j-1})}[\varphi(x)]}{t_{j+1} - t_{j-1}} = \bE_{x\sim\rho(t_j)}[\nabla\hat{s}(t_j, x) \cdot \nabla \varphi(x)]\,,\quad j = 0, \dots, K\,,
\end{align}
for all test functions $\varphi \in \mathcal{S}_k, k \in \{j-1, j, j+1\}$. Recall that $\nabla\hat{s}(t_j, \cdot)$ has minimal kinetic energy \eqref{eq:Prelim:KinEnergy} and so it has to be in $L^2(\rho(t_j))$, i.e. in $\hat{s}(t_j, \cdot) \in \mathcal{S}_j$ at all times.  We call $\nabla\hat{s}$ discretely compatible with $\{\rho(t_j)\}_{j=0}^K$  if \eqref{eq:DICE:DiscreteCompat} holds for $j = 0, \dots, K$.
Other time-discretization schemes can be used as well.

\paragraph{The DICE loss} We now show that the system of equations \eqref{eq:DICE:DiscreteCompat} is a system of Euler-Lagrange equations of the following quadratic functional $\LD: \mathcal{S} \to \mathbb{R}$, which we call the DICE loss.
    
    \begin{proposition}[DICE loss]
    \label{prop:DICEEulerLagrangeEq}
    A function $\hat{s}^* \in \mathcal S$ that minimizes 
    \begin{multline}
        \label{eq:Loss_DICE}
        \LD(s) = \sum_{j=1}^{K} \left( \frac{t_j - t_{j-1}}{2} \left( \bE_{x\sim\rho(t_{j})} \left[ \frac{1}{2} |\nabla s(t_j, x)|^2 \right] + \bE_{x\sim\rho(t_{j-1})} \left[ \frac{1}{2} |\nabla s(t_{j - 1}, x)|^2 \right] \right)\right. \\
        \left.- \frac{1}{2} \left( \bE_{x\sim\rho(t_{j})} \left[ s(t_j, x) + s(t_{j-1}, x) \right] - \bE_{x\sim\rho(t_{j-1})} \left[ s(t_j, x) + s(t_{j-1}, x) \right] \right) \right)
    \end{multline}
    solves the system of equations given in \eqref{eq:DICE:DiscreteCompat}.
    \end{proposition}
The proof can be found in \Cref{proof:DICEEulerLagrangeEq}. For a minimizer $\hat{s}^*$ of $\mathrm{L}_{\mathrm{DICE}}$, the behavior outside of the time points 
$\{t_j\}_{j=0}^K \subset [0,T]$ is not specified. 
A linear interpolant of the form
\begin{align}
        \label{eq:DICEPieceWiseLinearInTime}
        \hat{s}^*(t, x) \bigg |_{t \in [t_j, t_{j+1}]} = \frac{t_{j+1} - t}{t_{j+1} - t_j} \hat{s}^*(t_j, x) + \frac{t - t_j}{t_{j+1} - t_j} \hat{s}^*(t_{j+1}, x)\,, \quad j = 0, \dots, K-1\,,
\end{align}
is one way to extend the function to the entire time interval of $[0, T]$. Any other curve $t \mapsto \hat{s}(x,t)$ that intersects all $\{ \hat{s}^*(t_j, x) \}_{j=0}^K$ is also discretely compatible with $\{\rho(t_j)\}_{j = 0}^K$ and therefore a minimizer of $\mathrm{L}_{\mathrm{DICE}}$ over $\mathcal{S}$. The value of $\LD$ is independent of the specific form of the curve between the discrete time points.

\subsection{Properties of the DICE loss}\label{sec:DICE:Properties}
Recall that we are optimizing for a gradient field $\hat{s}$ which is defined only up to a constant because $\nabla \hat{s} = \nabla (\hat{s} + f)$ holds for functions $f: [0, T] \to \mathbb{R}$ that are constant over the spatial domain $\mathcal{X}$ but can vary over the time interval $[0, T]$. The following proposition shows that the DICE loss preserves this property in the sense that it is invariant under addition of functions that are constant over the spatial domain $\mathcal{X}$. 
    \begin{proposition}[Invariance of DICE loss to constants]\label{prop:InvariantF} The loss $L_{\text{DICE}}$ is invariant under additions of functions that are constant over the spatial domain: given a gradient field $s$ and
    any function $f: [0,T] \rightarrow \bR$ we have
    \begin{align}
         \mathrm{L}_{\mathrm{DICE}}(s + f) = \mathrm{L}_{\mathrm{DICE}}(s)\,.\end{align}
    \end{proposition}

\begin{proof}
    By direct substitution, we obtain the identity 
    \begin{align}
        \mathrm{L}_{\mathrm{DICE}}(s + f) &= \mathrm{L}_{\mathrm{DICE}}(s) - \frac{1}{2} \sum_{j=1}^K \left( \bE_{x\sim\rho(t_{j})} - \bE_{x\sim\rho(t_{j-1})} \right) \left[ f(t_{j}) + f(t_{j-1}) \right] \\
        &= \mathrm{L}_{\mathrm{DICE}}(s) - \frac{1}{2} \sum_{j=1}^K \underbrace{\left( f(t_{j}) + f(t_{j-1}) - f(t_{j}) + f(t_{j-1}) \right)}_{= \, 0}.
    \end{align}
    \end{proof}
    
The invariance to constants is a key property of the DICE loss, which is important for stable training; see Section~\ref{sec:CmpAM}. 
We now show existence of a minimizer $\hat{s}^*$ of the DICE loss $\LD$ in $\mathcal{S}$ and discuss in what sense it is unique, building on the Poincar\'e inequality of Assumption~\ref{assumption:prelim:poincare}. The invariance to constants of the DICE loss shown in \Cref{prop:InvariantF} plays a key role in the following proposition because it allows us to ignore constants and then application of the Poincar\'e inequality to show boundedness over $\mathcal{S}$ rather than only over functions with zero mean. We formulate the following proposition, which will lead to a corollary about the uniqueness of DICE minimizers. A proof can be found in \Cref{appdx:ExistenceProof}. 
    \begin{proposition}[The DICE loss admits a unique minimizer]
\label{prop:DICESolutionDiscreteInTime} 
Let \Cref{assumption:prelim:poincare} hold. Let further the density functions satisfy
    \begin{align}
    \label{eq:DICErho_log_deriv_is_L2}
        \frac{\rho(t_{j+1}, \cdot) - \rho(t_{j-1}, \cdot)}{\rho(t_{j}, \cdot)} \in L^2(\rho(t_j, \cdot))\,,
    \end{align}
    for $j = 0, \dots, K$. Consider now the loss $\LD^K: \mathcal{S}_0 \times \dots \mathcal{S}_K \to \mathbb{R}$ defined for $\bar{\hs} = [\hs_0, \dots, \hs_K] \in \mathcal{S}_0 \times \dots \times \mathcal{S}_K$ as
    \begin{multline}
    \label{eq:DICEDiscreteDecoupled}
        \LD^K(\bar{\hs}) = \sum_{j=1}^{K} \bigg( \frac{t_j - t_{j-1}}{2} \left( \bE_{x\sim\rho(t_{j})} \left[ \frac{1}{2} |\nabla \hs_j(x)|^2 \right] + \bE_{x\sim\rho(t_{j-1})} \left[ \frac{1}{2} |\nabla \hs_{j-1}(x)|^2 \right] \right) \\
        - \frac{1}{2} \left( \bE_{x\sim\rho(t_{j})} \left[ \hs_j(x) + \hs_{j-1}(x) \right] - \bE_{x\sim\rho(t_{j-1})} \left[ \hs_j(x) + \hs_{j-1}(x) \right] \right) \bigg)\,.
    \end{multline}
    Then, the problem
    \begin{align}
        \min_{\bar{\hs} \in \mathcal{S}_0 \times \dots \times \mathcal{S}_K} \LD^K(\bar{\hs})
    \end{align}
    admits a solution $\bar{\hs}^* = [\hs_0^*, \dots, \hs_K^*] \in \mathcal{S}_0 \times \dots \times \mathcal{S}_K$. In particular, there is a solution $\bar{\hs}_0^*$ in $\mathcal{S}_0^0 \times \dots \times \mathcal{S}_K^0$ that is unique with respect to  $\|\cdot\|_{\mathcal{S}_0 \times \cdots \times \mathcal{S}_K}$. 
    \end{proposition}
    Note that $\LD^K$ is simply $\LD$ restricted to functions defined over discrete sequences in time, rather than over a time continuum.
    We obtain the following two corollaries for the DICE loss, which is formulated as an optimization problem over functions of space and time.

\begin{corollary}
\label{corr:DICESolutionParametrizedInTime}
Let \Cref{prop:DICESolutionDiscreteInTime} apply. Furthermore, let $\hs^{*} \in \mathcal S$ be a minimizer of $\LD$ and let $\bar{\hs}^* = [\hs_0^*, \dots, \hs_K^*] \in \mathcal{S}_0 \times \dots \times \mathcal{S}_K$ be a minimizer of $\LD^K$. Then, it holds that
\begin{align}\label{eq:DICEDiscSolEqual}
    \|\nabla \hs^*(t_j, \cdot) - \nabla  \hs^*_j \|_{L^2(\rho(t_j))} = 0\,,\qquad j = 0, \dots, K\,.
\end{align}
\end{corollary}

\begin{corollary}
\label{corr:DICELowerBound}
    Let  \Cref{prop:DICESolutionDiscreteInTime} apply. The loss $\LD$ is lower bounded over $\mathcal{S}$. 
\end{corollary}

\section{The DICE loss in the infinite data limit}\label{sec:DICEInfinite}
In this section, we consider the error between the gradient field $\nabla\hat{s}^*$ given by a minimizer of the DICE loss over $\mathcal{S}$ and the true gradient field $\nabla s^*$. In \Cref{sec:Infinite:BoundedDensity}, we make assumptions on the boundedness of the density $\rho(0)$ at initial time $t = 0$ and discuss their implications.  In \Cref{sec:Infinite:DiscTime}, we then bound the error of the DICE minimizer at the discrete time points $\{t_j\}_{j = 0}^K$ at which the time marginals are available. The main result is presented in \Cref{sec:Infinite:InBetweenDiscTime} and establishes that the error between the DICE optimizer and the true gradient field is of first order with respect to the maximum distance between successive time points. We further show in \Cref{sec:Infinite:BoundInference} that the error of samples generated with the flow corresponding to the learned vector field is also first order with respect to the maximum distance between successive time points.  

\subsection{Bounded density functions from flow maps}\label{sec:Infinite:BoundedDensity}

We now make stronger assumptions on the density function so that we obtain constants that only depend on the initial density $\rho(0)$ as well as the vector field $u$ that generates the flow of the samples. Recall that we assume there exists a vector field $u$ such that $\partial_t \rho + \nabla \cdot (u \rho) = 0$.

\begin{assumption}[Upper and lower bounds on $\rho(0)$ for regular flows]
\label{asm:Smooth} 
There exist constants $\overline{\rho}_0 \geq \underline{\rho}_0 > 0$ such that the density function $\rho(0, \cdot): \mathcal{X} \to \mathbb{R}$ at time $t_0 = 0$ is bounded from above and below as 
\begin{equation}
    \underline{\rho}_0 \leq \rho(0, x) \leq \overline{\rho}_0\,,\qquad x \in \mathcal{X}\,.
\end{equation}
The flow map $\phi_t: \mathcal{X} \to \mathcal{X}$ given by the ODE \eqref{eq:Prelim:ODE} corresponding to the vector field $\nabla s^* = u$ is a twice continuously differentiable diffeomorphism for all $t \in [0, T]$ at any point $x \in \mathcal{X}$.
\end{assumption}

Recall that by our assumptions on $\mathcal X$, \Cref{asm:Smooth} implies that all $\rho(t)$ satisfy a Poincaré inequality with constant $\lambda(t) = \lambda_{\mathcal X} \min_{x \in \mathcal X} \rho(t, x)$ or better so that \Cref{assumption:prelim:poincare} is satisfied. \Cref{asm:Smooth} implies that the trajectories of the flow $\phi_t$ do not cross. Thus, given a point $x \in \mathcal{X}$, we can find a point $a \in \mathcal{X}$ such that $x = \phi_t(a)$ by inverting the flow map. Note that our convention for the flow map is such that it starts at time $t = 0$. We denote by $\phi_{t, \tau}: \mathcal{X} \to \mathcal{X}$ the map from time $\tau$ to time $t$,
    \begin{align}
        \phi_{t, \tau}(x) = \phi_t(\phi_{\tau}^{-1}(x)).
    \end{align}
The point $a \in \mathcal{X}$ can be interpreted as a label for the initial position of a particle. The point $\phi_t(a) \in \mathcal{X}$ is the position of the particle at time $t$. 

The regularity of the field $u$ controls the regularity of the flow map: If $u$ is $k$-times continuously differentiable in space $x$ and $\phi_t$ is defined for all $t \in [0,T]$, then the flow map $\phi_t$ is $k$-times continuously differentiable in space at all $t \in [0, T]$ and $t \mapsto D\phi_t(x)$ is $k+1$-times differentiable at all $x \in \mathcal{X}$; see \citet[Lemma 4.1.9]{abraham_manifolds_1988} for a proof. Building on the flow map $\phi_t$, the density $\rho(t, \cdot)$ at time $t$ is
\begin{align}
\rho(t, x) = \rho_0(\phi_t^{-1}(x)) \exp\left( - \int_0^t (\nabla \cdot u)(\tau, \phi_\tau(\phi_t^{-1}(x))) \, d\tau \right)\,,\qquad (t, x) \in [0, T] \times \mathcal{X}
\end{align}
by the method of characteristics. Through this construction, we see that bounds on $\rho(0, \cdot)$ imply bounds on $\rho(t, \cdot)$ over the whole time interval $[0, T]$, as the following lemma shows.

\begin{lemma}
\label{prop:dice:bounds_on_rho}
    Let \Cref{asm:Smooth} hold, then
    \begin{align}
        \rho(t, x) \leq \overline{\rho}_0 \exp\left( \overline{c} t \right), \quad \rho(t, x) \geq \underline{\rho_0} \exp\left( - \underline{c} t \right)\,,\qquad (t, x) \in [0, T] \times \mathcal{X}\,,
    \end{align}
where $\underline{c}$ and $\overline{c}$ are positive constants that depend on the divergence $\nabla \cdot u$ of $u$ and are finite when $u(t, \cdot)$ is continuously differentiable in space for all $t \in [0, T]$.
\end{lemma}

The bounds on $\rho$ mean that functions in $L^2(\rho(t_j))$ are also in $L^2(\rd x)$ for all $j$. This simplifies the following arguments.
    
\subsection{Error of DICE minimizer at discrete time points}\label{sec:Infinite:DiscTime}
Consider the marginals $\{\rho(t_j)\}_{j = 0}^K$ corresponding to $K+1$ time steps. Let $\hat{s}^* \in \mathcal{S}$ be a function that minimizes the DICE loss \eqref{eq:Loss_DICE}.  
The function $\hat{s}^*$ satisfies the time-discrete weak form of the continuity equation \eqref{eq:DICE:DiscreteCompat} at the $K$ time steps,
\begin{equation}\label{eq:DICE:Smooth:ContEqsHat}
 \int_{\mathcal X} \varphi(x) \hat \delta_{t_j} \rho(t_j, x) \, \rd x = \int_{\mathcal X} \nabla \varphi(x) \cdot \nabla \hat s^*(t_j, x) \rho(t_j, x) \, \rd x \,,\quad  \varphi \in H^1_0(\mathrm dx), j = 0, \dots, K\, ,
\end{equation}
where $H^1_0(\mathrm dx)$ is the set of functions $\varphi \in L^2(\mathrm dx)$ that have bounded derivative $\nabla \varphi \in L^2(\mathrm dx)$ such that $\nabla \varphi(x) \cdot n(x) = 0$ on the boundary of $\mathcal{X}$, where $n(x)$ denotes the unit normal. Notice that we use test functions in $H_0^1(\mathrm dx)$ now because the bounds on $\rho$ in \Cref{asm:Smooth} mean that functions in $L^2(\rho(t_j))$ are in $L^2(\mathrm dx)$ too; see \Cref{prop:dice:bounds_on_rho}.
The operator $\hat{\delta}_{t_j}$ denotes the  finite-difference approximations of the time derivatives 
\begin{equation}\label{eq:DICE:Smooth:DefDeltaRho}
\hat \delta_{t_j} \rho(t_j, \cdot) = \frac{\rho(t_{j+1}, \cdot) - \rho(t_{j-1}, \cdot)}{t_{j+1} - t_{j-1}}\,,\qquad j = 0, \dots, K\,.
\end{equation}
Because of the regularity assumptions on $\rho$ made in \Cref{asm:Smooth}, the density is strictly positive and bounded. This turns the continuity equation into an elliptic problem with homogeneous Neumann boundary conditions; see Section~\ref{sec:Cont}.

\begin{proposition}[Error bound for DICE solution at data time points]\label{prop:DICEErrorBoundAtDataTimepoints} Let \Cref{asm:Smooth} hold and recall that $\nabla s^*$ is the unique gradient field that is compatible with $\rho$. Then, for $j = 0, \dots, K$, and for $\hat{s}^*$ is a DICE minimizer over $\mathcal{S}$,
    \begin{align}\label{eq:DiscTimeBound}
        \| \nabla s^*(t_j, \cdot) - \nabla \hat s^*(t_j, \cdot) \|_{L^2(\rho(t, \cdot))} \leq \lambda_{\mathcal X}^{-1/2} \underline{\rho}_0^{-1} \exp\left( \underline{c} t_j \right) \| (\partial_t - \hat \delta_{t_j}) \rho(t_j, x) \|_{L^2(\rd x)}\,.
    \end{align}
\end{proposition}
The bound \eqref{eq:DiscTimeBound} only holds at the discrete time points $t_0, \dots, t_K$. Bounding the error  between time points $t \in [t_{j-1}, t_j]$ is the topic of the next subsection. 

\subsection{Error of DICE minimizer between discrete time points}\label{sec:Infinite:InBetweenDiscTime}
To bound the error of the gradient field $\nabla\hat{s}^*$ obtained with DICE with respect to the actual gradient field $\nabla s^*$, we make assumptions on the acceleration of the density $\rho$ in time, i.e., its second derivative in time. The reason is that in DICE, we approximate $\partial_t \rho$ by finite differences. The accuracy of this approximation depends on the second derivative of $\rho$ in time.

\begin{proposition}[Error of DICE minimizer at all times]
    \label{prop:DICE:BoundForAllT}
    Let \Cref{asm:Smooth} hold. Let $\hat s$ be the linear-in-time interpolant as defined in \eqref{eq:DICEPieceWiseLinearInTime} of $\hat s^*(t_0, \cdot), \ldots, \hat s^*(t_K, \cdot)$ for a DICE minimizer $\hat{s}^* \in \mathcal{S}$. If $\rho$ is twice continuously differentiable over $t \in (0, T)$ for all $x \in \mathcal X$, then
    \begin{align}\label{eq:PropBoundOverallContTime}
        \| \nabla s^*(t, \cdot) - \nabla \hat s(t, \cdot) \|_{L^2(\rho(t))} \leq C \Delta t_{\mathrm{max}}\,, \qquad t \in [0, T],
    \end{align}
    holds, where the maximal time-step size is $\Delta t_{\mathrm{max}} = \max_{j = 1, \dots, K} \{ t_j - t_{j-1} \}$. The constant $C > 0$ can be upper bounded as
    \begin{equation}
    C \leq 3 \lambda^{-1/2}_{\mathcal X} \, \overline{\rho_0} \exp\left( \overline{c} T \right) \left( C_{\partial_t^2\rho} + \overline{\rho}_0 \exp \left( \overline{c} T \right) \left( \max_{t \in [0,T]} \| \partial_t \rho(t) \|_{L^2(\rd x)} \right ) \max_{t \in [0,T]} \|\partial_t \rho(t)\|_{L^{\infty}(\mathrm dx)} \right) \,,
    \end{equation}
    with $C_{\partial_t^2\rho} = \max_{t \in [0,T]} \| \partial_t^2 \rho(t) \|_{L^2(\rd x)};$
    thus, in particular, $C$ depends on derivatives of $\rho$ and $s^*$ only and is independent of the time-step sizes $t_j - t_{j - 1}$ for $j = 1, \dots, K$. 
    \end{proposition}

The proposition leads to a corollary for the error of the DICE minimizer $\hat{s}^*$ that is not necessarily linearly interpolating between the time points $\{t_j\}_{j = 0}^K$. 

\begin{corollary}
\label{corr:DICE:BoundForAllT_NNDiscretization}
Let \Cref{prop:DICE:BoundForAllT} apply and recall that $\hat s^* \in \mathcal{S}$ is a DICE minimizer. If $\hat{s}^*$ is Lipschitz in $t$ with respect to $\|\cdot\|_{L^2(\rho(t))}$ with constant $L^t_{\hat s^*} > 0$, 
    \begin{align}
        \| \nabla \hat s^*(t, \cdot) - \nabla \hat s^*(\tau, \cdot) \|_{L^2(\rho_t)} \leq L^t_{\hat s^*} |t - \tau| \qquad \forall t, \tau \in [0, T]\,,
    \end{align}
then the bound \eqref{eq:PropBoundOverallContTime} remains true for $\hat{s}^*$ and not only the linear interpolant $\hat{s}$, 
\begin{equation}
\label{eq:PropBoundOverallContTime_NNDiscretization}
        \| \nabla s^*(t, \cdot) - \nabla \hat s^*(t, \cdot) \|_{L^2(\rho(t))} \leq C(L^t_{\hat s^*}) \Delta t_{\mathrm{max}}\,, \qquad t \in [0, T] \, ,
\end{equation}
except that now the constant $C(L^t_{\hat{s}^*}) > 0$ depends on the Lipschitz constant $L^t_{\hat{s}^*}$ of $\hat{s}^*$. 
\end{corollary}

Because \Cref{prop:DICE:BoundForAllT} holds with a constant $C$ that can be upper bounded independent of the number of time steps $K$ and the time-step sizes $t_{j} - t_{j - 1}$, we can take the limit so that $\max_j |t_{j} - t_{j - 1}| \to 0$. Let $\hat{s}_{\Delta t_{\text{max}}}^*$ be a DICE minimizer for $\Delta t_{\text{max}} = \max_j |t_j - t_{j - 1}|$. We then  obtain convergence in the sense 
    \begin{align}
        \|\nabla \hat{s}_{\Delta t_{\text{max}}}^*(t, \cdot) - \nabla {s}^*(t, \cdot)\|_{L^2(\rho(t))} \to 0 \text{ for } \Delta t_{\text{max}} \to 0\,, \qquad t \in [0,T].
    \end{align}
This result also motivates taking the limit of $\Delta t_{\text{max}} \to 0$ on the DICE loss $\LD$ itself, which recovers \eqref{eq:Cont:ContLossddtau} as long as the time integral exists.

\subsection{Bound for the inference error}\label{sec:Infinite:BoundInference}
Let \Cref{asm:Smooth} hold and, furthermore, assume that $\hat{s}^*$ is a DICE minimizer. The following proposition bounds the error between the law $\hat{\rho}$ obtained from the continuity equation with the field $\nabla \hat{s}^*$ and the law $\rho$ corresponding to the actual gradient field $\nabla s^*$ in the 2-Wasserstein metric $W_2$. The key step in the proof is a Gr\"onwall-type argument. 

\begin{proposition}[DICE inference error]
\label{prop:bound_on_inference_error_w2}
Let \Cref{asm:Smooth} apply and let $\hat{s}^*$ be a twice continuously differentiable DICE minimizer so that \Cref{prop:DICE:BoundForAllT} holds. Let $\hat{\rho}$ be the density function obtained from $\nabla \hat{s}^*$ via the continuity equation in weak form \eqref{eq:Prelim:ContWeakForm} with $\hat{\rho}(0, \cdot) = \rho(0, \cdot)$. Then, the difference between $\rho(t, \cdot)$ and $\hat{\rho}(t, \cdot)$ in the $W_2$ metric is bounded as
\begin{align}
    W_2(\rho(t, \cdot), \hat \rho(t, \cdot))
    \leq \left( W_2(\rho(0, \cdot), \hat \rho(0, \cdot)) + C(L^t_{\hat s^*}) t \Delta t_{\text{max}}
    \right) \mathrm e^{ \int_0^t L_{\hat s^*}^x(\tau) \, \rd \tau}
\end{align}
where $C(L^t_{\hat s^*})$ denotes the constant in \eqref{eq:PropBoundOverallContTime_NNDiscretization} in \Cref{corr:DICE:BoundForAllT_NNDiscretization} and the constant $L_{\hat s^*}^x(t)$ is  
\begin{align}
L_{\hat s^*}^{x}(t) =  \sup_{x \in \mathcal X} \Vert D^2_x \hat s^*(t, x) \Vert_{\mathrm{op}}\,,
\end{align}
with $D^2_x \hat s^*(t, x)$ denoting the Hessian of $\hat{s}^*$ with respect to $x$. 
    \end{proposition}

    There are two ways to control $L_{\hat s^*}^x(t)$. First, in practice, $\hat s^*$ will be represented in parametrized form such as a neural network, which can be chosen to limit the norm of $D^2 \hat s^*$. Second, bounds on the difference $D^2 s^*(t_j, \cdot) - D^2 \hat s^*(t_j, \cdot)$ in H\"older norms can be derived when $\rho$ and the domain are regular enough (by tools from Schauder theory, c.f. \citet[Chapter 6]{gilbarg_elliptic_2001}). Estimates at time $t_j$ can then be extended to $t \in [t_j, t_{j+1}]$ as we did for $\nabla s^*(t, \cdot) - \nabla \hat s^*(t, \cdot)$ in \Cref{prop:DICE:BoundForAllT}. We leave a thorough investigation of this for future work.

\section{Comparing DICE to Action Matching}\label{sec:CmpAM}
We compare training with the DICE loss to the AM loss. We first describe the AM loss in \Cref{sec:AM:AMLoss}. We then show in \Cref{sec:AM:ResidualTerm} that the DICE loss avoids residual terms that can emerge in the time-discrete AM loss and destabilize the training by growing arbitrarily large. A concrete example where AM is unstable is introduced in \Cref{sec:blowup_example}.

\subsection{The AM loss function}\label{sec:AM:AMLoss}
The AM loss is derived by \cite{neklyudov_action_2023} with partial integration in time from the dynamical transport problem introduced by \cite{benamou_computational_2000}, and reads
    \begin{align}
        \label{eq:OV_Loss}
        \cAM(\hat s) 
        &= \int_0^T \mathbb{E}_{x \sim \rho(t)} \left[ \frac{1}{2} |\nabla \hat s(t, x)|^2 + \partial_t \hat s(t, x) \right] \rd t - \mathbb{E}_{x \sim \rho(t)} \left[ \hat s(t, x) \right] \bigg|_{t=0}^T.
    \end{align}
In the time-continuous formulation, the AM loss controls the continuous counterpart of the error that is controlled by the minimizers of the DICE loss (\Cref{corr:DICE:BoundForAllT_NNDiscretization}), 
\begin{equation}\label{eq:AM:Error}
   \int_0^T \|\nabla s^*(t, x) - \nabla \hat s(t, x)\|_{L^2(\rho(t))}\mathrm dt = \cAM(\hat s) + C(\nabla s^*)\,,
   \end{equation}
   where $C(\nabla s^*)$ is a constant that only depends on the actual gradient field $\nabla s^*$ but is independent of $\hat s$ \citep{neklyudov_action_2023}. 
   
   The time-discrete AM loss  is
    \begin{align}
        \label{eq:AM_Loss}
        \LAM(\hat s) = \sum_{j=0}^{K} w_j \, \E_{x \sim \rho(t_j)} \left[ \left( \partial_t \hat s + \frac{1}{2} | \nabla \hat s |^2 \right) (t_j, x) \right] - \E_{x \sim \rho(t)} \left[ \hat s(t, x) \right] \bigg |_{t = 0}^{T}.
    \end{align}
    Here the integral in time is numerically computed with a quadrature rule given by $K+1$ nodes $\{(t_j, w_j) \}_{j = 0}^{K}$; this may be, for example, a Monte Carlo estimator as in \cite{neklyudov_action_2023}, or it may be a deterministic rule as in \cite{berman2024parametric}.  The time-discrete AM loss controls the error \eqref{eq:AM:Error} then up to the quadrature error $|\cAM(\hat s) - \LAM(\hat s)|$.

\subsection{The DICE loss avoids unbounded residual terms that can emerge in the empirical AM loss during training}\label{sec:AM:ResidualTerm}
Recall that our problem statement from \Cref{sec:Prelim:ProblemFormulation} determines a field $\hat s$ only up to a function $f: [0,T] \to \mathbb{R}$ because $\nabla (\hat s + f) = \nabla \hat s$ for any $f$ that is constant in space $\mathcal{X}$ but can vary with time $t$. 
Notice that, for such $f$, the time-continuous AM loss satisfies $\cAM(\hat s + f) = \cAM(\hat s)$.
The DICE loss, which is time-discrete by definition, is also invariant in that $\LD(\hat s + f) = \LD(\hat s)$ for such $f$; see \Cref{prop:InvariantF}.
However, the time-discrete AM loss $\LAM$ can violate this invariance.
    \begin{proposition}\label{prop:AM:BadConstant}
        The time-discrete AM Loss is not necessarily invariant with respect to addition of a spatially constant but time varying function $f: [0, T] \to \bR$. It holds that
        \begin{align}
        \label{eq:AM:LossInvariance}
            \LAM(\hat s + f) = \LAM(\hat s) + R(f)\,,
        \end{align}
        where
        \begin{align}\label{eq:AM:ResTerm}
            R(f) = \sum_{j=0}^{K} w_j \, \partial_t f(t_j) - f(T) + f(0)\,.
        \end{align}
    \end{proposition}

    \begin{proof}
    We substitute $\hat s + f$ into \eqref{eq:AM_Loss}: 
    \begin{multline}
        \LAM(\hat s + f) = \sum_{j=0}^{K} w_j \, \E_{x \sim \rho(t_j)} \left[ \left( \partial_t s + \partial_t f + \frac{1}{2} | \nabla \hat s + \nabla f|^2 \right) (t_j, x) \right] \\ 
            - \E_{x \sim \rho(t)} \left[ (\hat s + f)(t, x) \right] \bigg |_{t = 0}^{T}\,.
    \end{multline}
    Because $\nabla f = 0$, we obtain 
    \begin{align}
        \LAM(\hat s + f) = \LAM(\hat s) + \sum_{j=0}^{K} w_j \, \partial_t f(t_j) - f(T) + f(0)\,,
    \end{align}
    as claimed.
    \end{proof}

Because of the residual term $R(f)$, the quadrature error in the time-discrete AM loss can be made arbitrarily large, which means that $|\cAM(\hat s) - \LAM(\hat s)|$ gets large and thus minimizing the time-discrete AM loss $\LAM(\hat s)$ does not reflect the error \eqref{eq:AM:Error} well anymore. When $f$ is a function with sharp gradients, the second, spurious term of the discrete AM loss can dominate the first term in \eqref{eq:AM:LossInvariance}. This can lead to a negative feedback loop that destabilizes the training with the AM loss; as shown 
in the next section. The residual term $R(f)$ given in \eqref{eq:AM:ResTerm} depends on the quadrature rule used for the time integral.  \cite{berman2024parametric} investigate higher-order quadrature rules in time to keep the residual term small; however, it does not vanish. 

We illustrate this in \Cref{fig:t_tau_dep_of_s}. We view the function $t \mapsto \bE_{x \sim \rho(t)}[s(t, x) + f(t)]$ in terms of the map $(\tau,t) \mapsto \bE_{x \sim \rho(\tau)}[s(t, x) + f(t)]$ evaluated at point $\tau=t$ for a potentially rough $f$. Then, the function $\tau \mapsto \bE_{x \sim \rho(\tau)}[s(t, x) + f(t)]$ for fixed $t$ is a regular function, controlled by the evolution of samples. In contrast, $t \mapsto \bE_{x \sim \rho(\tau)}[s(t, x) + f(t)]$, for fixed $\tau$, can be an arbitrarily rough function when the value of the constant $f(t)$ is not controlled. The DICE loss only depends on map $\tau \mapsto \bE_{x \sim \rho(\tau)}[s(t, x) + f(t)]$ (see \Cref{eq:Cont:ContLossddtau}). In contrast, the AM loss, through the $\partial_t s$ term therein, depends on $t \mapsto \bE_{x \sim \rho(\tau)}[s(t, x) + f(t)]$.

\begin{figure}[t]
    \centering
    \tdplotsetmaincoords{70}{110} 
    \begin{tikzpicture}[tdplot_main_coords, scale=2]
    
    \draw[thick,->] (0,0,0) -- (2,0,0) node[anchor=north east]{$t$};
    \draw[thick,->] (0,0,0) -- (0,2,0) node[anchor=north west]{$\tau$};
    \draw[thick,->] (0,0,0) -- (0,0,2) node[anchor=south]{$\bE_{x \sim \rho(\tau)}[s(t, x) + f(t)]$};
    
    \draw[black, thick, domain=0:2, samples=100, variable=\y]
        plot ({0}, {\y}, {1 - 0.5*sin(deg(2*pi*\y))*cos(deg(pi*\y/3))});

    \draw[black, thick, domain=0:2, samples=100, variable=\y]
        plot ({\y}, {0}, {1 + 0.5*\y*cos(deg(pi*\y/4))*sin(deg(3*pi*\y))*cos(deg(7*pi*\y))});
        

    \node at (1.0, -1.5, 0.8) {$t \mapsto \bE_{x \sim \rho(\tau)}[s(t, x) + f(t)]$};
    \node at (0.75, 2.9, 1.1) {$\tau \mapsto \bE_{x \sim \rho(\tau)}[s(t, x) + f(t)]$};
    
    \end{tikzpicture}
    \caption{The DICE loss approximates the typically smoother derivative $\frac{\rd}{\rd \tau}\bE_{x \sim \rho(\tau)}[s(t, x) +f(t)]$ of the map $\tau \mapsto \bE_{x \sim \rho(\tau)}[s(t, x) + f(t)]$, which keeps $f(t)$ fixed. In contrast, the AM relies on the function $t \mapsto \bE_{x \sim \rho(\tau)}[s(t, x) + f(t)]$ with varying $f$ because of the term $\bE_{x \sim \rho(\tau)}[\partial_t (s(t, x) + f(t))]$ in the AM loss.}
    \label{fig:t_tau_dep_of_s}
\end{figure}
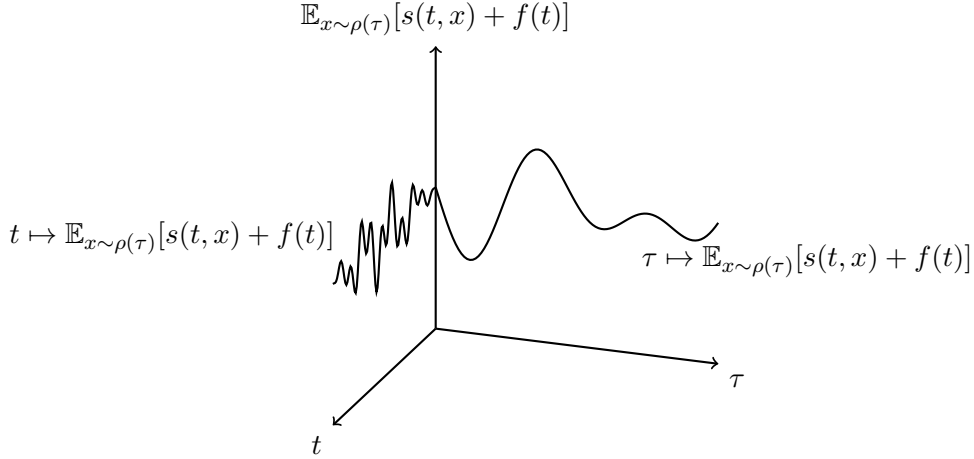

\subsection{The DICE loss avoids unboundedness of the AM loss}
\label{sec:blowup_example}
Consider a transport couple $(\rho, \nabla s)$ that solves the continuity equation \eqref{eq:Prelim:ContWeakForm} over $[0, T]$.  Consider now the case that we have data at two time steps only, namely at time $t_0 = 0$ and $t_1 = T$. A discretely compatible couple $(\rho, \hat{s})$ with $\{\rho(0), \rho(T)\}$ is given by the linear interpolant
\begin{align}
\hat{s}(t, x) = \frac{T - t}{T} s(0, x) + \frac{t}{T} s(T, x)\,,
\end{align}
which is a minimizer of the DICE loss.
In this special case and assuming that the time derivative of $\hat{s}$ can be exactly computed, the linear interpolant $\hat{s}$ is also a minimizer of the time-discrete AM loss $\LAM$ with trapezoidal quadrature rule because 
\begin{align}
    \LAM(\hat s) &= \frac{T}{2} \left( \bE_{x \sim \rho(0)}\left[ \frac{1}{2} |\nabla \hat s(0, x)|^2 + \partial_t \hat s(0, x) \right] + \bE_{x \sim \rho(T)} \left[ \frac{1}{2} |\nabla \hat s(T, x)|^2 + \partial_t \hat s(T, x) \right] \right) \nonumber \\ &\quad - \bE_{x \sim \rho(t)}[\hat s(t, x)] \bigg|_{t=0}^T \\
    &= \frac{T}{2} \left( \bE_{x \sim \rho(0)}\left[ \frac{1}{2} |\nabla \hat s(0, x)|^2 \right] + \bE_{x \sim \rho(T)} \left[ \frac{1}{2} |\nabla \hat s(T, x)|^2 \right] \right) \nonumber \\ &\quad + \frac{1}{2} \left( \bE_{x \sim \rho(0)} + \bE_{x \sim \rho(T)} \right) \left[ \hat s(T, x) - \hat s(0, x) \right] - \bE_{x \sim \rho(T)}[\hat s(T, x)] + \bE_{x \sim \rho(0)}[\hat s(0, x)] \\
    &= \frac{T}{2} \left( \bE_{x \sim \rho(0)}\left[ \frac{1}{2} |\nabla \hat s(0, x)|^2 \right] + \bE_{x \sim \rho(T)} \left[ \frac{1}{2} |\nabla \hat s(T, x)|^2 \right] \right) \nonumber \\ &\quad - \frac{1}{2} \left( \bE_{x \sim \rho(T)} - \bE_{x \sim \rho(0)} \right) \left[ \hat s(T, x) + \hat s(0, x) \right] \\
    &= \LD(\hat s).
\end{align}

Let us now consider a specific parametrization of $\hat{s}$ that depends on the parameter $\theta$ as 
\begin{align}
\hat{s}_{\theta}(x, t) = \hat{s}(x, t) + f_{\theta}(t)\,,
\end{align}
where the second term is the time-dependent function
\begin{align}
f_{\theta}(t) = - \frac{\theta t^2}{T^2} \left(1 - \frac{t}{T}\right)\,.
\end{align}
The function $f_{\theta}$ leads to a residual term $R(f_{\theta})$ in the time-discrete AM loss; see \Cref{sec:AM:ResidualTerm}.
If we now substitute $\hat{s}_{\theta}$ into the time-discrete AM loss, we obtain
\begin{align}
\LAM(\hat{s}_{\theta}) = \LD(\hat{s}_{\theta}) + \frac{1}{2} \left( \bE_{x \sim \rho(0)}\left[  (\partial_t f_{\theta})(0) \right] + \bE_{x \sim \rho(T)} \left[ (\partial_t f_{\theta})(T) \right] \right)
\end{align}
which can be written as
\begin{equation}\label{eq:AM:GradientIssue1}
\LAM(\hat{s}_{\theta}) = \LD(\hat{s}_{\theta}) - \frac{\theta}{2T}.
\end{equation}
Notice that the DICE loss appears in \eqref{eq:AM:GradientIssue1}. Let us now take the gradient of the DICE loss and set it to zero: $\nabla_{\theta} \LD(\hat{s}_{\theta}) = 0.$ 
At this point, however,
\begin{align}
\nabla_{\theta} \LAM(\hat{s}_{\theta}) = - \frac{1}{2T}\,,
\end{align}
causing a gradient descent method to drive $\theta$ away from neighborhood of the desired optimum. Note that adding a regularization term of the form $\left( \mathbb E_{x\sim\rho(t)}\left[ s(t, x) \right] \right)^2$ to the definition of $\LAM$ to control the un-determined constant close to zero $\mathbb E_{x\sim\rho(t)}\left[ s(x,t) \right] \approx 0 \; \forall t$ does not avoid divergence either. The argument in fact remains unchanged if $\{ \hat{s}(t_j, x) \}_j$ satisfy $\mathbb E_{x\sim\rho(t_j)}\left[ \hat s(t_j, x) \right] = 0$ for all discrete time points $t_0, t_1, \dots, t_K$ at which data are available.

\section{Training with the DICE loss on samples}\label{sec:DICESamples}
So far we have assumed that the time marginal densities are available to us
for the discrete time points $t_0, t_1, \dots, t_K$.
We now consider the case where we only have samples of the time marginals at these
points. We introduce the fully empirical DICE loss in \Cref{sec:Training:AllDiscrete} and emphasize that no access to the density values and not even sample trajectories are required. 
We furthermore describe how the DICE loss can be extended to allow for modulation with respect to physical parameters in \Cref{sec:DICE:ParamCase} and describe an entropic variant in \Cref{sec:Training:Entropy}.

\subsection{Learning vector fields from samples of marginals at discrete time points with the DICE loss}\label{sec:Training:AllDiscrete}
We now consider the situation that we have available only samples from the time marginals $\{\rho(t_j)\}_{j = 0}^K$ at the discrete time points $\{t_j\}_{j = 0}^K$. 
We have a data set \begin{align}\label{eq:Prelim:Dataset}
        \mathcal{D} = \{ X_i(t_j) \quad \,|\, \quad i = 1, \dots, N_j \,,\quad j = 0, \dots, K\} \subseteq \mathcal{X}\,,
     \end{align}
which contains $N_j$ samples $X_1(t_j), \dots, X_{N_j}(t_j)$ from $\rho(t_j)$ for $j = 0, \dots, K$.
The number of samples $N_j$ can vary for different time points $t_j$. 
Furthermore, we do not assume the availability of trajectory data: there is no pairing of data points at different times; in particular $X_i(t_j)$ and $X_i(t_{j + 1})$ are \emph{not} viewed as coming from a single trajectory $X_i: [0,T] \to  \mathcal{X}.$

The fully empirical DICE loss is obtained by replacing all expectation values in \eqref{eq:Loss_DICE} by Monte Carlo estimators, for $j = 0, \dots, K$, 
using the data \eqref{eq:Prelim:Dataset}, to obtain
\begin{equation}\label{eq:PlainMonteCarlo}
\hat{\mathbb{E}}_{x \sim \rho(t_j)} [g(x)] = \frac{1}{N_j}\sum_{i = 1}^{N_j} g(X_i(t_j))\,,
\end{equation}
where $g$ is either $\nabla s(t_j, \cdot)$ or $s(t_j, \cdot).$ 
We explicitly define the fully discrete DICE loss for completeness: 
\begin{multline}
        \label{eq:Loss_DICE_fulldisc}
        \hLD(s) = \sum_{j=1}^{K} \bigg( \frac{t_j - t_{j-1}}{2} \bigg( \hat \bE_{x\sim\rho(t_{j})} \left[ \frac{1}{2} |\nabla s(t_j, x)|^2 \right]
        + \hat \bE_{x\sim\rho(t_{j-1})} \left[ \frac{1}{2} |\nabla s(t_{j - 1}, x)|^2 \right] \bigg) \\
        - \frac{1}{2} \left( \hat \bE_{x\sim\rho(t_{j})} \left[ s(t_j, x) + s(x, t_{j-1}) \right] - \hat  \bE_{x\sim\rho(t_{j-1})} \left[ s(t_j, x) + s(t_{j-1}, x) \right] \right) \bigg)\,.
\end{multline}
Replacing the expectations with empirical estimators in $\LD$ to obtain $\hLD$ introduces additionally errors that depend on the number of samples.

\subsection{Generalization of the DICE loss to parametrized processes}\label{sec:DICE:ParamCase}
Let us now consider processes $X(t; \mu)$ that additionally to time $t$ depend a parameter $\mu \in \Qcal \subseteq \mathbb{R}^{d_q}$. Correspondingly, the law $\rho(t; \mu)$ and the vector field $u: \mathcal{X} \times [0, T] \times \Qcal \to \mathbb{R}$ depend on $\mu$ as well.
We consider the case where we have a data set $\mathcal{D}$ that contains realizations of samples over time and parameters,
    \begin{align}
    \mathcal{D} = \{X_i(t_j; \mu_l)\quad \,|\, j = 0, \dots, K\,, \quad i = 1, \dots, N_j\,,\quad \quad l = 0, \dots, M\}\,,
    \end{align}
    with $M \in \mathbb{N}$ training parameters $\mu_1, \dots, \mu_M \in \mathcal{Q}$. 
To learn a parameterized potential field $\hat{s}: \mathcal{X} \times [0, T] \times \Qcal \to \mathbb{R}$ from $\mathcal{D}$, we extend the DICE loss to the parametric case as
    \begin{align}
    \LD^{\mu}(\hat{s}) = \sum_{l = 1}^M \LD(\hat{s}(\cdot, \cdot; \mu_l))\,.
    \end{align}

\subsection{The DICE loss with an entropy term}\label{sec:Training:Entropy}
An entropy term can be added to play the role of a regularization term to the DICE loss
    \begin{multline}
        \label{eq:Loss_DICE_entropic}
        \hLD^\varepsilon(s) = \sum_{j=1}^{K} \bigg( \frac{t_j - t_{j-1}}{2} \bigg( \bE_{x\sim\rho(t_{j})} \left[ \frac{1}{2} \bigl |\nabla s(t_j, x)\bigl |^2 + \frac{\varepsilon^2}{2} \Delta s(t_j, x) \right] \\
        + \bE_{x\sim\rho(t_{j-1})} \left[ \frac{1}{2} \bigl |\nabla s(t_{j - 1}, x)\bigl |^2 + \frac{\varepsilon^2}{2} \Delta s(t_{j-1}, x) \right] \bigg) \\
        - \frac{1}{2} \left( \bE_{x\sim\rho(t_{j})} \left[ s(t_j, x) + s(x, t_{j-1}) \right] - \bE_{x\sim\rho(t_{j-1})} \left[ s(t_j, x) + s(t_{j-1}, x) \right] \right) \bigg)\,,
    \end{multline}
    where $\epsilon > 0$ serves as a regularization parameter. Similar regularizers have been used for other loss functions by \cite{neklyudov_action_2023} and \cite{berman2024parametric} for AM. 
The corresponding inference equation then changes from the continuity equation to the Fokker-Planck equation $\partial_t \rho + \nabla \cdot (\rho \nabla s) = \frac{1}{2}\varepsilon^2 \Delta \rho$. When $\hs^{*,\varepsilon}$ is a minimizer of $\hLD^\varepsilon$, then, formally, $\partial_t \rho + \nabla \cdot (\rho \nabla \hs^{*,\varepsilon}) = \frac{1}{2 }\varepsilon^2 \Delta \rho$ and therefore new samples can be generated from $\rho$ by solving the stochastic differential equation
\begin{align}
        \rd \widehat X(t) = \nabla \hs^{*,\varepsilon}(t, \widehat X(t)) \rd t + \varepsilon \rd W_t, \quad X(0) \sim \rho(0),
\end{align}
where $\rd W_t$ denotes a Wiener process; instead of the ODE given in \eqref{eq:Prelim:ODESampling}. 
The optimal value of $\varepsilon$ is a hyper-parameter of the scheme that needs to be tuned. 

\section{Numerical experiments}\label{sec:NumExp}
We now apply DICE to a variety of examples. We consider a toy example with a constant potential in \Cref{ssec:DvsAM} to demonstrate the improved training stability of DICE over AM. In \Cref{sec:NumExp:KnownPotential}, we consider a problem with a known potential to show how the error of the DICE potential behaves. We then apply DICE to propagating random waves (\Cref{sec:NumExp:RandomWaves}), Vlasov-Poisson instabilities (\Cref{sec:NumExp:Vlasov}), and a nine-dimensional setting of the chaotic Lorenz '96 model (\Cref{sec:NumExp:Rayleigh}) to demonstrate that DICE is widely applicable and enables learning models of population dynamics that can generate high-quality sample population with low inference costs.

In all of the resulting numerical experiments, we train using the Adam optimizer with a learning rate of $2 \times 10^{-3}$ with a cosine learning rate scheduler. Unless otherwise noted, the batch size is 256 samples over 256 time points and we use a $128$ width multi-layer perceptron of six layers with a swish activation function; see Appendix~\ref{appdx:DetailsNumExp} for further details.

The code used can be found at \texttt{https://github.com/ToBlick/DICE}.

\sisetup{detect-weight=true}
\begin{table}[t]
  \centering
  \resizebox{0.8\linewidth}{!}{
  \begin{tabular}{l | ll | ll | ll | ll}
  \toprule
  example: & \multicolumn{2}{c|}{\textbf{two-stream}} & \multicolumn{2}{c|}{\textbf{bump-on-tail}}  & \multicolumn{2}{c|}{\textbf{strong Landau}} & \multicolumn{2}{c}{\textbf{9D chaos}} \\
  \midrule
    metric:     & e.e.  & r.t.\,[s]  & e.e.  & r.t.\,[s] & e.e.  & r.t.\,[s] & sinkhorn  & r.t.\,[s] \\
  \midrule
    CFM   
    & 1.44  &  139 
    & 5.52  &  141 
    & 1.96 & 238
    & 0.259  & 36\\  
    NCSM & 0.245 & 1142 
    & 0.626 & 1133 
    & 3.58 & 6753
    & 0.869 & 1109 \\ 
    AM  & 0.275 & 6  
    & 0.892 & 6 
    & NaN & -
    & 80.1 & 7 \\ 
    HOAM  & 0.078 & 6 
    & 0.427 & 6
    & 0.784 & 8
    & 0.214 & 7 \\ 
    DICE (ours) & 0.070 & 6 
    & 0.283 & 6
    & 0.735 & 8
    & 0.200 & 7 \\ 
    
    \bottomrule
  \end{tabular}
  }
  \caption{The DICE loss is more stable to train than the AM loss in these examples. Additionally, generating sample populations with DICE (inference) can be orders of magnitude faster than with diffusion- and flow-based modeling  approaches (NCSM, CFM) that have to condition on time to generate population trajectories.}
  \label{tbl:comparisons}
\end{table}

\begin{figure}[t]
    \centering 
    \begin{subfigure}{0.45\textwidth} \includegraphics[width=\linewidth]{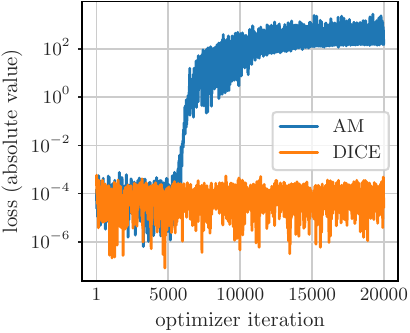}
    \caption{absolute loss over optimization}
    \end{subfigure}%
    \hfill
    \begin{subfigure}{0.55\textwidth}
    \includegraphics[width=1.10\linewidth]{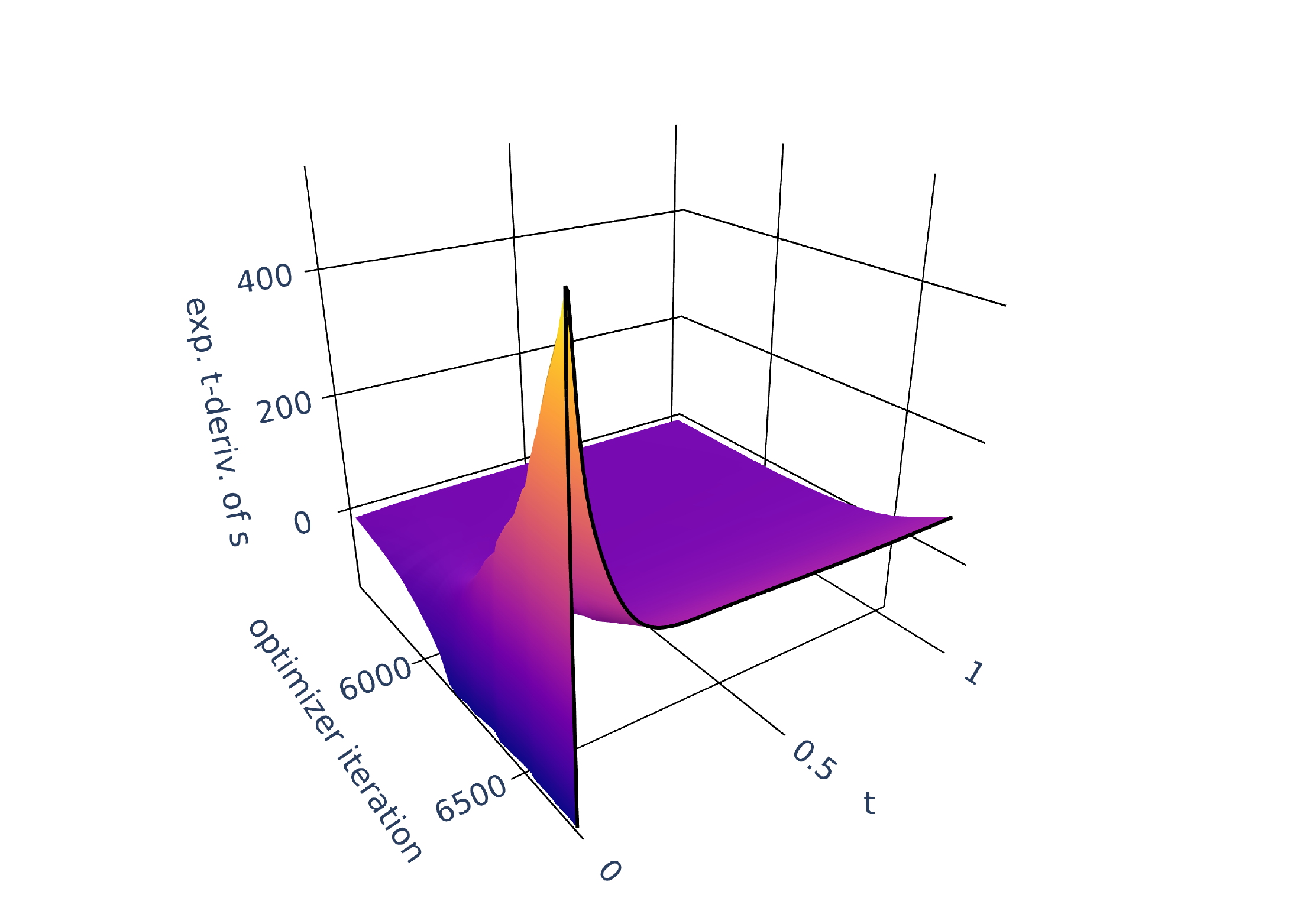}
    \caption{$\hat{\mathbb E}_{\rho(t)}[\partial_t s_{\theta_n}(t, \cdot)]$ throughout training}
    \end{subfigure}%
    \caption{Toy example: The DICE loss \eqref{eq:Loss_DICE} leads to stabler training behavior (left) because it is invariant to spurious constants that can emerge during training (see Proposition~\ref{prop:InvariantF}). In contrast, the empirical AM loss \eqref{eq:AM_Loss} can diverge to $-\infty$ by adding a function that is constant in space but varies rapidly in time (right), which leads to unstable training. The $z$-axis plots an empirical estimate $\hat{\mathbb E}_{\rho(t)}[\partial_t s_{\theta_n}(t, \cdot)]$ over the optimization iterations as the training becomes unstable.}
    \label{fig:AM_loss_over_iterations}
\end{figure}
    
\subsection{Demonstration of stabler training behavior of DICE compared to AM}
\label{ssec:DvsAM}
We demonstrate the stable training behavior of DICE compared to AM on a toy example using a stationary process $X(t) \sim \mathcal{N}(0, 10^{-2})$ so that $\rho(t)$ is the density function of the normal random variable with mean 0 and variance $10^{-2}$ for all times $t \in [0, 1]$. Correspondingly, any gradient field $\nabla s$ with $s$ constant over the spatial coordinate $x$ leads to a continuity equation that maintains the initial density $\rho(0)$. We take a time-step size of $\delta t = 2^{-8}$ and generate $N_j = 10^4$ realizations from $\rho(t_j)$ at all $K = 512$ times $j = 0, \dots, 512$.  Notice that because we sample independently for all $t_j$, there are no sample trajectories in this example, i.e., sample $X_i(t_j)$ and $X_i(t_{j+1})$ are uncorrelated for all $i = 1, \dots, N_j$ and $j = 0, \dots, K$. We parametrize $s$ as $s_{\theta}$ with a multi-layer perceptron architecture with two hidden layers of width 32. The loss minimization is done with the ADAM optimizer with mini batching in the samples ($n_x = 128)$ and time points $(n_t = 128)$. The used learning rate goes from $5\times 10^{-4}$ to $10^{-6}$ on a cosine schedule over $2 \times 10^4$ iterations.
    
Let us first consider the AM loss $\LAM$ given in \eqref{eq:AM_Loss} with a composite trapezoidal rule in time, which is in agreement with the work by \cite{berman2024parametric} that proposes to use quadrature in time rather than Monte Carlo estimation. Expectation over the space variables are estimated with plain Monte Carlo \eqref{eq:PlainMonteCarlo}. \Cref{fig:AM_loss_over_iterations}a plots the empirical AM loss over the optimization iterations. After about 5000 iterations, the absolute value of the AM loss starts to increase, which indicates a training instability: The inaccurate integration of the time derivative of $s$ in the loss \eqref{eq:AM_Loss} leads to an amplification of numerical errors. As discussed in Section~\ref{sec:AM:ResidualTerm} and shown in Proposition~\ref{prop:AM:BadConstant}, such an instability can emerge when the optimizer tries to push the residual term \eqref{eq:AM:ResTerm} towards $-\infty$ by optimizing for a function $f$ that is challenging to integrate numerically. 

To see that this is indeed happening in this example, consider $t \mapsto \mathbb{E}_{x \sim \rho(t)} [ \partial_t s_{\theta_n}(t, x) ]$, which is the expectation of the field $s_{\theta_n}$ at optimization iteration $n$. The expectation is taken over $x \sim \mathcal{N}(0, 1)$. The function $t \mapsto \mathbb{E}_{x \sim \rho(t)} [\partial_t s_{\theta_n}(t, x) ]$ needs to be numerically integrated in time to obtain the empirical AM loss \eqref{eq:AM_Loss}, and this leads to the residual term \eqref{eq:AM:ResTerm}. \Cref{fig:AM_loss_over_iterations}b shows an empirical estimate of the function, which starts to form a sharp kink at optimization iteration $n$ of about 5000. This sharp kink is challenging to numerically integrate, which leads to a large absolute residual \eqref{eq:AM:ResTerm} that causes the absolute value of the loss function to grow as shown in \Cref{fig:AM_loss_over_iterations}a. Notice that while the absolute value of residual grows, the residual has a negative sign and so by forming a shaper and sharper kink, the optimizer makes the loss smaller and smaller. In contrast, if we use the DICE loss \eqref{prop:DICEEulerLagrangeEq} to optimize for a field $s_{\theta}$, then we obtain the more stable loss curve shown in Figure~\ref{fig:AM_loss_over_iterations}a.  Furthermore, because the population dynamics are constant in this example, we expect that the norm of $\nabla \hat{s}_{\theta_n}$ is close to zero, which is indeed what we see to almost single precision in Figure~\ref{fig:AM_loss_over_iterations}a because the DICE loss is close to zero and it includes the term $\sum_j w_j \hat{\mathbb{E}}_{x \sim \rho(t_j)} [\frac12 | \nabla s_{\theta_n}(t_j, x) |^2 ]$. 

\subsection{Example with known potential}\label{sec:NumExp:KnownPotential}
\begin{figure}
\begin{subfigure}{0.48\textwidth}
\includegraphics[width=0.9\columnwidth]{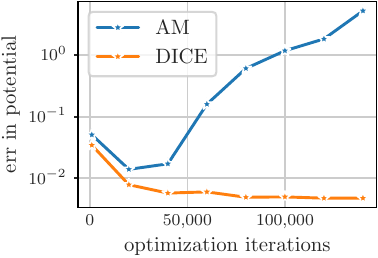}
\caption{error of the learned potential}
\end{subfigure}
\hfill
\begin{subfigure}{0.48\textwidth}
\includegraphics[width=0.9\columnwidth]{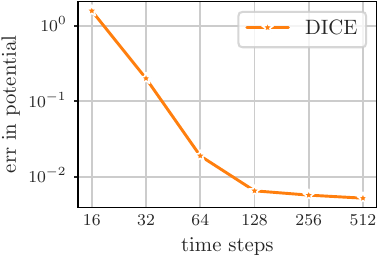}\caption{error versus number training samples in time}
\end{subfigure}
\caption{Known potential: Training with the DICE loss leads to stabler training behavior compared to the AM loss in this example. We observe a decrease in error as we increase the batch size in time $n_t$. At around $n_t = 128$, the error is mostly dominated by the error in $n_x$, which is held constant at $n_x = 256$.}
\label{fig:AnalyticExampleErrorPlot}
\end{figure}

In this experiment, we consider a process $X(t)$ that follows the time-dependent potential 
\begin{align}
    V(t, x) = \sin \left(\frac \pi 2 t \right)^2 V_1(x) + \cos \left(\frac \pi 2 t \right)^2 V_0(x)
\end{align}
with
\begin{align}
    V_0(x) = \sum_{i=1}^d \left( \sin(x_i) + \cos(x_i) + x_i^2 + x_i \right), \quad
    V_1(x) = \sum_{i=1}^d \left( x_i^4 - 16 x_i^2 + 5 x_i \right)\,.
\end{align}
Having the potential $V$ given analytically allows us to estimate the error 
$$ \int_t \bE_{x \sim \rho(t)} \left[ | \nabla \hat s(t, \cdot) - \nabla V(t, \cdot) |^2 \right] \mathrm dt,$$
which is the error of $\hat{s}$ that is also considered in Proposition~\ref{prop:DICE:BoundForAllT}. We consider the relative error of the potential, which we estimate as  
\begin{align}
    \sum_{j=0}^K w_j \, \frac{ \hat{\mathbb E}_{x \sim \rho(t_j)} \left[ | \nabla \hat s(t_j, \cdot) - \nabla V(t_j, \cdot) |^2 \right] }{ \hat{\mathbb E}_{x \sim \rho(t_j)} \left[ | \nabla V(t_j, \cdot) |^2 \right]  }
\end{align}
on the training data time steps $t_0, \dots, t_K$, where the $w_0, \dots, w_K$ denote quadrature weights for the composite trapezoidal integration.

We draw $N_0 = 2048$ samples from $\rho_0 = \mathcal N(0, 1)$ where $d = 2$ and set the time-step size to $\Delta t = 1/512$. The rest of the setup is as described at the beginning of \Cref{sec:NumExp}.  Figure~\ref{fig:AnalyticExampleErrorPlot} shows that the DICE loss allows us to accurately infer the time-dependent gradient field $\nabla V$. The example shows that DICE does indeed recover the sample dynamics in cases where they follow a time-dependent gradient flow. In contrast, the AM loss exhibits the same instability that was observed in the example in subsection \ref{ssec:DvsAM} for high optimization iteration numbers.

\begin{figure}
    \centering 
    \begin{subfigure}{\textwidth} \includegraphics[width=\linewidth]{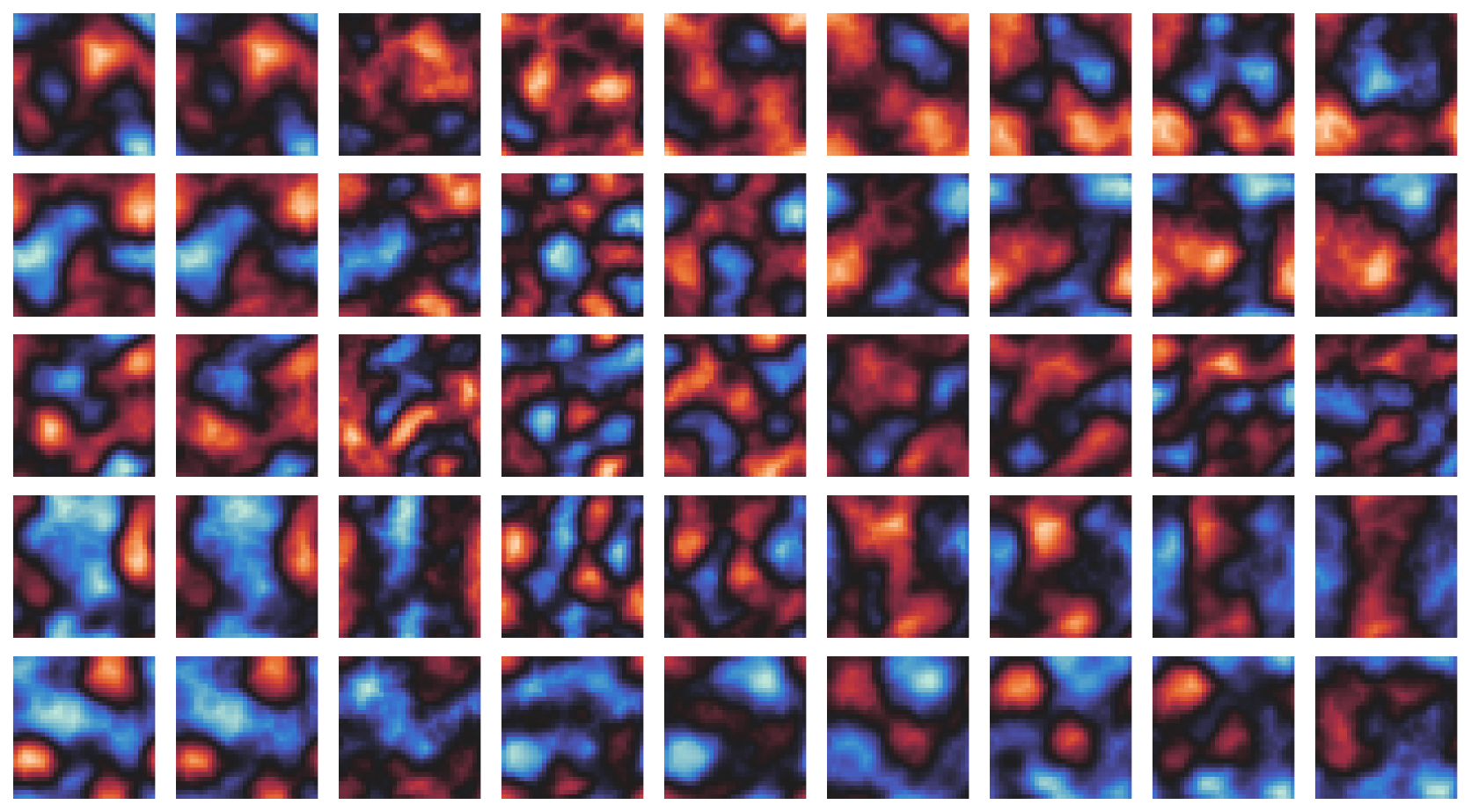}
    \caption{true samples}
    \end{subfigure} \\%
    \begin{subfigure}{\textwidth}
    \includegraphics[width=\linewidth]{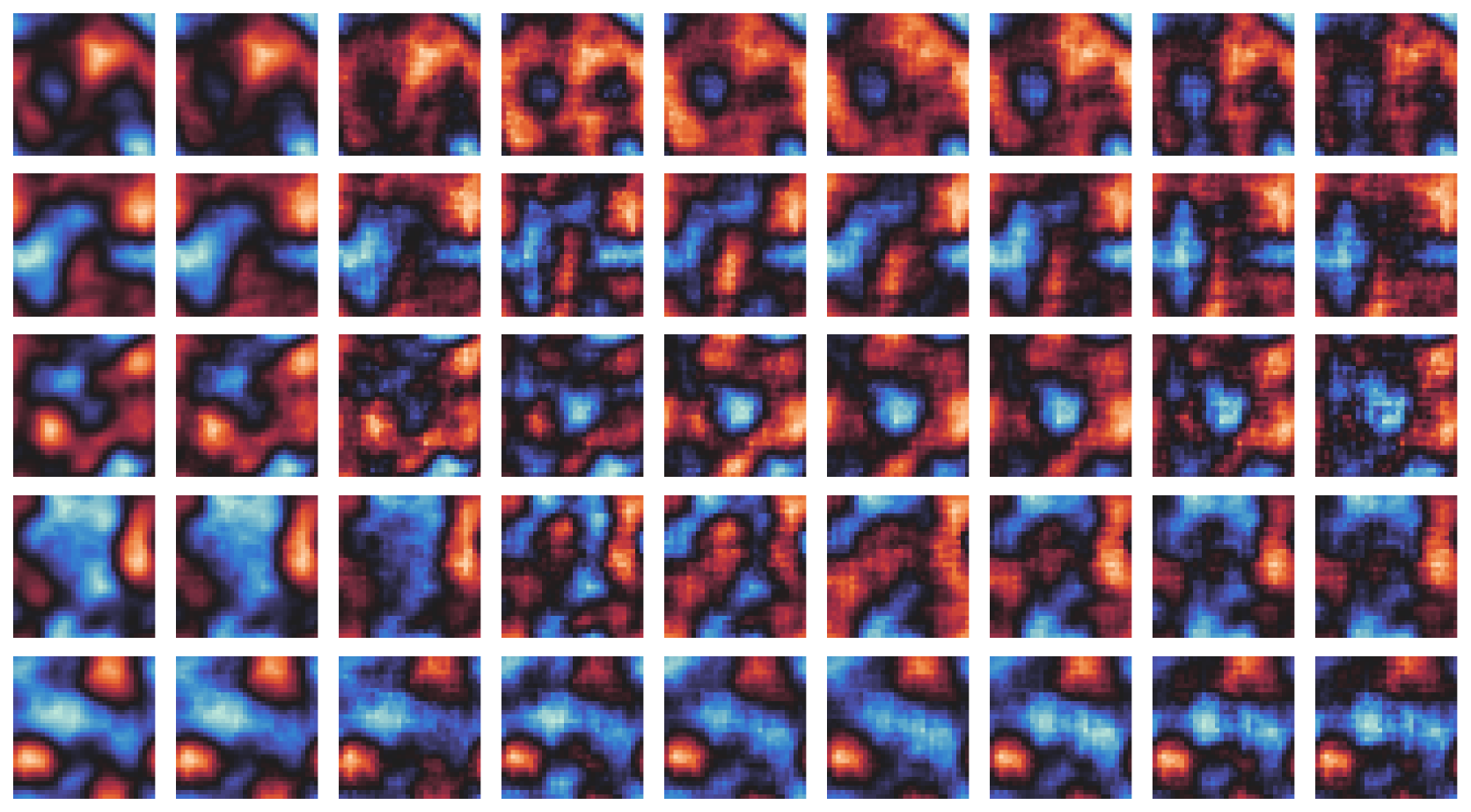}
    \caption{DICE samples}
    \end{subfigure}%
    \caption{Random waves: Marginal trajectories for five different initial conditions (top to bottom) over time (left to right). DICE captures the dynamics on a population level and thus the generated samples (bottom) are similar to the true samples (top) only in distribution but not necessarily in a point-wise sense.}
    \label{fig:wave:samples}
\end{figure}

\subsection{Random waves}\label{sec:NumExp:RandomWaves}
We now consider the wave equation, using random initialization to create a distribution. We learn the resulting population dynamics and test the learned model via its ability to predict moments. 

\paragraph{Setup} Consider the wave equation $\partial_t^2 \phi(t, x) + \Delta \phi(t, x) = 0$ in two spatial dimensions $x \in [0,2\pi)^2$, with periodic boundary conditions in $x$, over the time interval $t \in [0, 8]$. The initial condition on
$\phi$  is generated by choosing random Fourier coefficients from a log-normal distribution which peaks at a wave number $k=1$ and which controls the regularity of the random field. The time-derivative $\partial_t \phi(0, x)$ is set to zero at all $x$. We discretize the wave equation in space with a Fourier spectral method. We discretize in time with second-order finite differences and set the time-step size to $\delta t = 5 \times 10^{-3}$. We solve on a grid of $256 \times 256$ and then down-sample to $32 \times 32$ points to obtain the state $X(t)$ of dimension $d = 1024$. We generate $N_j = 2048$ training data samples at all time steps $j = 0, \dots, K$ and then train with DICE using $10^5$ optimization steps. Once we have trained $s_{\theta}$ we use it to generate $i = 1, \dots, 1024$ new samples $\widehat X_i(t_1), \dots, \widehat X_i(t_K)$ over the same time steps $t_1, \dots, t_K$ with initial conditions $\widehat X_i(0)$ sampled as described above. While at first it might look futile to  approximate the second-order dynamics of the wave equation with first-order ones \eqref{eq:Prelim:ODE}, we note that in the first-order dynamics we allow the right-hand side function $\nabla s = u$ to vary with time $t$, whereas in the second-order dynamics the right-hand side function is constant in time.

\paragraph{Results}
Recall that DICE captures the dynamics on a population level and thus the generated samples $\widehat X_i(t_j)$ that follow \eqref{eq:Prelim:ODE} with $\nabla s_{\theta}$ are similar to the  physical trajectory $t \mapsto X(t)$ only on a population level. In Figure~\ref{fig:wave:moments} we show the state $X(t)$ plotted in a two-dimensional domain as a heat plot for five test realizations of the initial condition and next to them the corresponding generated samples $\widehat X(t)$. They clearly differ in a point-wise sense. We stress again that DICE is not capturing sample dynamics but population dynamics and thus the generated samples can look differently when compared point-wise. 

To assess the quality of the generated samples, we need to consider quantities of interest that are determined by the distribution of $X(t)$ and $\widehat X(t)$ rather than their sample dynamics. For this, we consider the average $\langle X(t) \rangle$ over the components of $X(t)$, which has to stay constant $\langle X(t) \rangle = \langle X(0) \rangle$ over time $t$ because we impose periodic boundary conditions and mass is conserved. 
More generally, we also consider the moments $\langle X^i(t) \rangle$ of higher order $i \in \mathbb{N}$. \Cref{fig:wave:moments} shows estimates of the first, second, and third moment obtained with the generated samples with DICE, which are in close agreement with the estimated moments obtain from a test set. The shaded area shows 1/10 of a standard deviation and is also in close agreement. 

Recall that DICE aims to match the population dynamics while at the same time minimizes the kinetic energy. Thus, sample populations generated with DICE should have a lower kinetic energy, even though they are in agreement on a population level with the original system. We now empirically show  that the generated population indeed takes a path of less energy compared to the training data, while accurately approximating quantities of interest that depend on the population alone. We estimate the integrated kinetic energy over time as 
    \begin{align}
        E_{\mathrm{kin}}(\{ X_i(t_j) \}_{i,j}) = \sum_{j=1}^K \frac{1}{2 N_j}\sum_{i=1}^{N_j}   \frac{\| X_i(t_{j}) - X_i(t_{j-1})\|^2_2}{t_{j} - t_{j-1}},
    \end{align}
    and report it for a test data set and generated samples in \Cref{fig:wave:Ekin}. The gap shows that DICE sample populations have a lower kinetic energy than the true sample populations, which is in agreement with the objective of DICE to minimize the kinetic energy.

\begin{figure}[t]
    \centering 
    \begin{subfigure}{0.3\textwidth} \includegraphics[width=\linewidth]{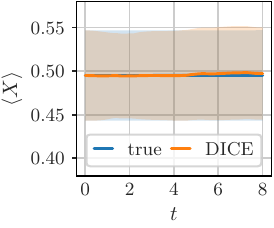}\caption{mean}
    \end{subfigure}%
    \hfill
    \begin{subfigure}{0.3\textwidth}
    \includegraphics[width=\linewidth]{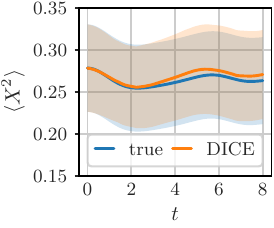}
    \caption{second moment}
    \end{subfigure}%
    \hfill
    \begin{subfigure}{0.3\textwidth}
    \includegraphics[width=\linewidth]{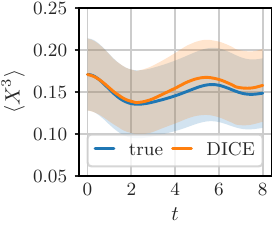}
    \caption{third moment}
    \end{subfigure}%
    \caption{Random waves: While the sample trajectories generated by DICE can look different than the original samples when compared point-wise, the generated sample population captures accurately quantities that depend on the population rather than the samples such as moments.} 
    \label{fig:wave:moments}
\end{figure}

\begin{SCfigure}
\includegraphics[width=0.45\linewidth]{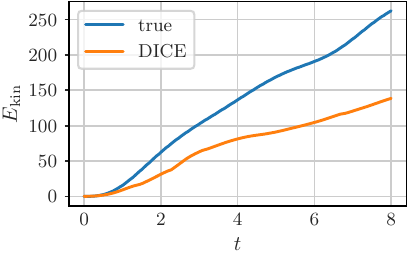}
    \caption{Random waves: The sample population generated by DICE has a lower kinetic energy than the original samples, which is indented because DICE aims to minimize the kinetic energy so that samples have to move as little as possible while matching the original samples in terms of population.} 
    \label{fig:wave:Ekin}
\end{SCfigure}

\subsection{Vlasov-Poisson instabilities}\label{sec:NumExp:Vlasov}
We now learn the population dynamics of particles in the strong Landau damping experiment and show that samples generated with DICE accurately predict the growth of electric energy as the instability grows. 

\paragraph{Setup}
Let us consider the Vlasov-Poisson system that describes the interaction of charged particles. The particle density $\rho: [0, T] \times \mathcal{X}_x \times \mathcal{X}_v \to \mathbb{R}$ is defined on the product of spatial
domain $\mathcal{X}_x = [0, 4\pi) \times [0, 1) \times [0, 1)$ and velocity domain $\mathcal{X}_v =\mathbb{R}^3$ domain. Periodic boundary conditions are applied in the spatial domain.
The dynamics of the density are governed by the following nonlinear and nonlocal
evolution equation:
\begin{align*}
\partial_t \rho(t, x, v) &= -v \cdot \nabla_x \rho(t, x, v) - \nabla \phi(t, x) \cdot \nabla_v \rho(t, x, v)\,, \\
-\mu^2 \Delta \phi(t, x) &= 1 - \int_{\mathcal{X}_v} \rho(t, x, v)\mathrm dv\,,
\end{align*}
where $\mu$ is the mass of the charged particles. We consider the strong Landau damping experiment, with initial condition
\begin{align}
\rho(0, x, v) = \frac{1}{\sqrt{2\pi}^3} \left (1 + \alpha \cos \left (2\pi \frac{x_1}{l_1} \right ) \right ) \exp\left(-\frac{|v|^2}{2} \right).
\end{align}
The parameters $\mu \in \{0.5, \dots, 1.5\} \setminus \{1.0\}$ are used for generating training data and $\mu = 1.0$ is kept for testing. The perturbation $\alpha$ is set to $0.25$. The time-step size is $\delta t = 0.05$ and the time interval is $[0, 8.75]$. Data are generated with the Struphy package introduced by \citet{nielsen_high-order_2023} that includes a detailed description of the strong Landau damping experiment.\footnote{\url{https://gitlab.mpcdf.mpg.de/struphy}} We store $10^5$ marker particles from the simulation and use $N_j = 25000$ of them as training data for our method over all times steps $j = 0, \dots, K$. We integrate in time over $t \in [0, 12.5]$ with time-step size $0.05$.

\paragraph{Results}
In \Cref{fig:Landau:two_fullwidth_images}, we show slices of the distribution function on phase space at time $t = 6$, in the form of two-dimensional histogram plots. Filaments in the projection onto the $x_1, v_1$-dimension are  visible and captured well by the DICE sample populations. We also compare to noise-conditioned score matching (NCSM) \citep{song2020generativealg} and conditional flow matching (CFM) \citep{lipman2022flow,albergo2023building}, which condition on time to generate marginal trajectories. The same architecture and training procedure are used as for DICE. As \Cref{fig:Landau:two_fullwidth_images} shows, NCSM and CFM achieve a comparable accuracy in the histogram plots. However, the inference runtimes reported in  Table~\ref{tbl:comparisons} show that the inference step of DICE is orders of magnitude faster than the inference with  methods that condition on time; see also the discussion by \cite{neklyudov_action_2023} and \cite{berman2024parametric}. Recall that DICE provides one trajectory over the whole time range $[0, T]$ per inference step, whereas conditioning on time requires one inference step per time step $t_0, \dots, t_K$. Notice that in Table~\ref{tbl:comparisons} we also compare to higher-order action matching (HOAM) introduced by \cite{berman2024parametric}, which is less prone to training instabilities by using the action matching loss with Simpsons quadrature in time; however, this imposes the restriction that the time steps $t_0, \dots, t_K$ are equidistant in time. Furthermore, AM with higher-order quadrature can still suffer from instabilities when training for extended times (see \Cref{ssec:DvsAM}); in this example, we utilize manual early stopping criteria.

\Cref{fig:Landau:ee_two_fullwidth_images_ee} shows the evolution of the electric energy. The generated DICE sample population is able to accurately capture the variation over four orders of magnitude for the majority of the observed time interval. Using the AM loss for training in this problem is challenging: we
were unable to obtain accurate approximations of the gradient fields despite extensive hyper-parameter sweeps; this observation is in agreement with \cite{berman2024parametric}.

In Table~\ref{tbl:comparisons}, we also report results obtained with DICE on the two-dimensional bump-on-tail and two-stream Vlasov-Poisson instabilities. Similar observations as for the strong Landau experiment holds for the bump-on-tail and two-stream instabilities. Appendix~\ref{appdx:2DVlasov} shows the corresponding histogram plots and electric energy curves.

\begin{figure}
  \centering
  \includegraphics[width=\textwidth]{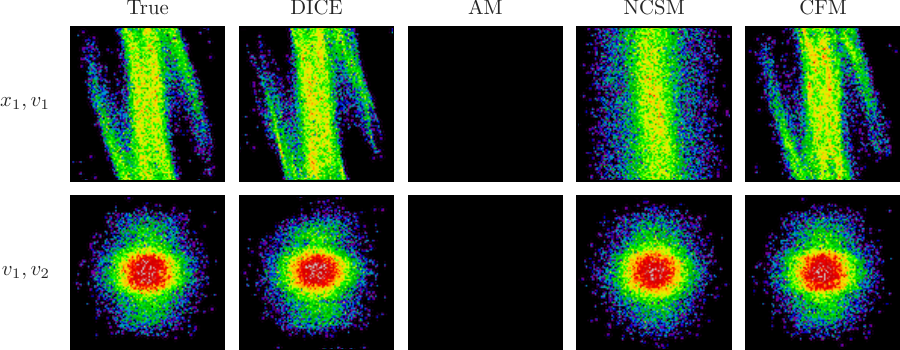}
  \caption{Vlasov-Poisson problem: The histogram plots of particles show that DICE sample populations match well the original sample populations, whereas training AM on this problem is challenging.}
  \label{fig:Landau:two_fullwidth_images}
\end{figure}

\begin{figure}
\includegraphics[width=\textwidth]{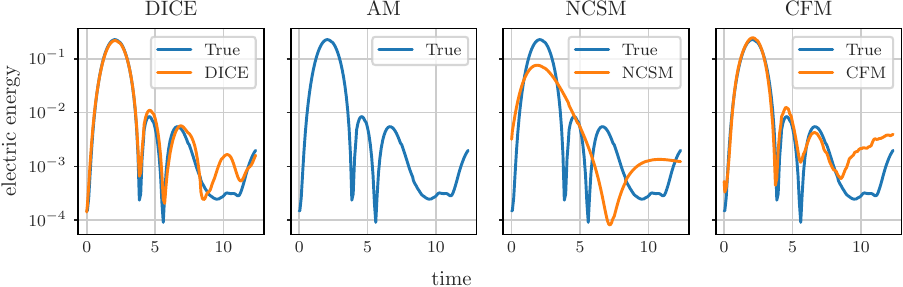}
\caption{Vlasov-Poisson problem: Sample populations generated with DICE accurately capture the electric energy in this example. }
\label{fig:Landau:ee_two_fullwidth_images_ee}
\end{figure}

\subsection{Rayleigh–B\'{e}nard convection}\label{sec:NumExp:Rayleigh}

We consider the system of nine ODEs, introduced by \cite{PeterReiterer_1998} as an approximate description of Rayleigh–B\'{e}nard convection -- a flow that is set in motion by a density gradient. The system can exhibit chaotic behavior, analogous to the three dimensional Lorenz system \citep{DeterministicNonperiodicFlow}. We choose the initial condition to be a nine-dimensional Gaussian random variable with zero mean and standard deviation $2 \times 10^{-2}$. To generate data, we integrate the system up to time $T = 20$ with an explicit Euler time integrator and time-step size $10^{-2}$. We study parametric dependence with respect to the parameter $\mu$, the Rayleigh number. The training data consists of $N_j = 25000$ samples at each
observation time $t_j$, and for parameter $\mu$ the set $\{13.5, 13.6, \dots, 14.2\}$. The test parameters are $\mu \in \{13.65, 14.05\}$. 

\Cref{fig:Lorenz} shows projections of the $9$-dimensional distribution functions in the form of histogram plots. Training with the AM loss fails in this example because of instabilities. In contrast, training with the  DICE loss leads to a gradient field that can generate accurate sample populations that match the true sample populations well. To quantitatively compare the DICE sample population quality, we estimate the time-averaged Sinkhorn divergence between the DICE population and the true population, which is low (better) compared to time-conditioned sampling with NCSM and CFM. Furthermore, generating sample populations with DICE is about one order of magnitude faster than with NCSM and CFM. 

\begin{figure}
  \centering
  \includegraphics[width=\textwidth]{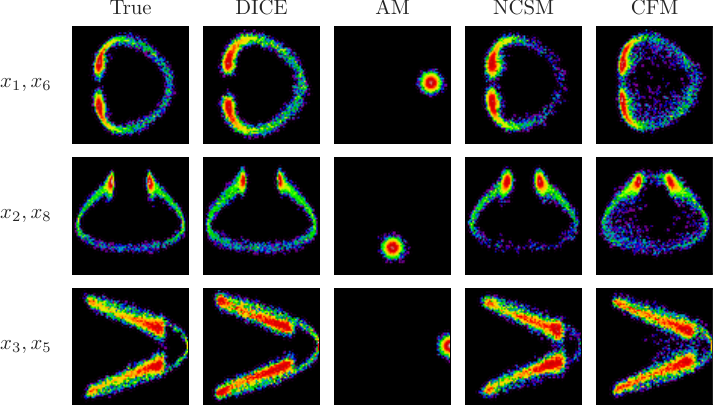}
  \caption{Chaotic system: Even though sample dynamics of this chaotic system can be challenging to learn, DICE can capture well the smoother population dynamics. The results shows that the sample populations generated with DICE accurately match original population.}
  \label{fig:Lorenz}
\end{figure}

\section{Conclusions}\label{sec:Conc}
We introduced DICE for learning population dynamics from samples of stochastic processes. By directly inverting the continuity equation through a carefully designed empirical loss, DICE avoids costly simulation-based training, bypasses solving nonlinear optimal transport problems, and supports fast inference without conditioning on time. A key contribution is the formulation of the DICE loss function that is well-posed even in discrete time and leads to more stable training behavior than other loss functions for learning population dynamics. Numerical results show that DICE enables efficient generation of sample trajectories that are consistent with the evolving population law, even in complex or chaotic systems. More broadly, DICE underscores the potential of learning population dynamics for reduced and surrogate modeling as an alternative to classical sample-level modeling, especially in settings where sample trajectories are unavailable or difficult to predict.

\acks{The first and second authors have been partially funded by the Air Force Office of Scientific Research (AFOSR), USA, award FA9550-21-1-0222 and FA9550-24-1-0327. The fourth author was partially supported by the Office of Naval Research, USA, under award N00014-22-1-2728. The third author is funded by a Department if Defense, USA, Vannevar Bush Faculty Fellowship.}

\bibliography{dice}

\appendix

\section{Proofs}

\subsection{The DICE loss is a variational formulation of the continuity equation}
\begin{proof}[\Cref{prop:DICEEulerLagrangeEq}]
    \label{proof:DICEEulerLagrangeEq}
    Consider $s^\varepsilon$ such that $s^\varepsilon(t_j, x) = \hat{s}^*(t_j, x) + \varepsilon \varphi_j(x)$ with $\varphi_j \in \mathcal C^\infty_0(\mathcal X)$ for all $j$. If $\hat{s}^*$ is optimal for \eqref{eq:Loss_DICE}, then it has to hold that
        \begin{align}\label{eq:DICE:DICEMinimizer:EqB}
            \frac{\rd}{\rd \varepsilon} \mathrm{L}_{\mathrm{DICE}}(s^\varepsilon) \bigg |_{\varepsilon = 0} = 0\,, \qquad \text{ for all } \varphi_j \in \mathcal C^\infty_0(\mathcal X) \text{ and  } j = 0, \dots, K.
        \end{align}
        We now write for brevity
        \begin{align}
            \bE_j = \bE_{x\sim\rho(t_{j})}, \; \hat{s}_j^* = \hat{s}^*(t_j, \cdot), \text{ and } (\bE_{j+1} - \bE_{j})[s] = \bE_{j+1}[s] - \bE_{j}[s].
        \end{align}
        Because
        \begin{align}
            \frac{\rd}{\rd \varepsilon} \bE_j \left[ \frac{1}{2} |\nabla s^\varepsilon_j |^2 \right]
            &= \frac{\rd}{\rd \varepsilon} \bE_j \left[ \frac{1}{2} |\nabla \hat{s}^*_j + \varepsilon \nabla \varphi_j|^2 \right] \\
            &= \bE_j \left[ \nabla \hat{s}^*_j \cdot \nabla \varphi_j + \varepsilon |\nabla \varphi_j|^2 \right]
        \end{align}
        and
        \begin{align}
            \frac{\rd}{\rd \varepsilon} \bE_j \left[ s^\varepsilon_j \right] = \bE_j \left[ \varphi_j \right],
        \end{align}
        we find
        \begin{multline}\label{eq:DICE:DICELossEqA}
            \frac{\rd}{\rd \varepsilon} \mathrm{L}_{\mathrm{DICE}}(s^\varepsilon) \bigg |_{\varepsilon = 0} 
            = \sum_{j=1}^{K} \bigg( \frac{t_j - t_{j-1}}{2} \left( \bE_j \left[ \nabla \hat{s}^*_j \cdot \nabla \varphi_j \right] + \bE_{j-1} \left[ \nabla \hat{s}^*_{j-1} \cdot \nabla \varphi_{j-1} \right] \right)  \\
            - \frac{1}{2} \left( \bE_j - \bE_{j-1} \right) \left[ \varphi_j + \varphi_{j-1} \right] \bigg)\,,  
        \end{multline}
        for all $\varphi_j  \in \mathcal C^\infty_0(\mathcal X)$ and for $j = 0, \dots, K$. 
        
        Now let $j^* = 1, \dots, K - 1$ and set $\varphi_j$ to the constant zero $\varphi_j \equiv 0$ function for all $j \in \{0, \dots, K\}\setminus \{j^*\}$. Recall that \eqref{eq:DICE:DICELossEqA} has to be zero for all $\varphi_j  \in \mathcal C^\infty_0(\mathcal X)$ and thus it has to be zero also for our particular choice of test functions. We therefore obtain that for all $j^* = 1, \dots, K - 1$, the following equation has to hold for all $\varphi_{j^*} \in C^\infty_0(\mathcal X)$ 
        \begin{multline}
            0 = \frac{t_{j^*} - t_{j^*-1}}{2}  \bE_{j^*} \left[\nabla \hat{s}^*_{j^*} \cdot \nabla \varphi_{j^*} \right]  - \frac{1}{2} \left( \bE_{j^*} - \bE_{j^* - 1} \right) \left[ \varphi_{j^*} \right] \\
            + \frac{t_{j^*+1} - t_{j^*}}{2}  \bE_{j^*} \left[ \nabla \hat{s}^*_{j^*} \cdot \nabla \varphi_{j^*} \right]  - \frac{1}{2} \left( \bE_{j^* +1 } - \bE_{j^*} \right) \left[ \varphi_{j^*} \right]\,,
        \end{multline}
        which can be re-written as
        \begin{align}
            \frac{1}{t_{j^*+1} - t_{j^*-1}} (\bE_{j^* + 1} - \bE_{j^*})\left[ \varphi_{j^*} \right] = \bE_{j^*} \left[\nabla \hat{s}^*( t_{{j^*}}, \cdot) \cdot \nabla \varphi_{j^*} \right] ,
        \end{align}
        which is equation $j^*$ in the system of equations given in \eqref{eq:DICE:DiscreteCompat}.
        For $j^* \in \{0, K\}$, we find
        \begin{align}
            0 &= \frac{t_{1} - t_{0}}{2} \left( \bE_{0} \left[ \nabla \hat{s}^*_{0} \cdot \nabla \varphi_{0} \right] \right) - \frac{1}{2} \left( \bE_{1} - \bE_{0} \right) \left[ \varphi_{0} \right], \\
            0 &= \frac{t_{K} - t_{K-1}}{2} \left( \bE_{K} \left[ \nabla \hat{s}^*_{K} \cdot \nabla \varphi_{K} \right] \right) - \frac{1}{2} \left( \bE_{K} - \bE_{K-1} \right) \left[ \varphi_{K} \right],
        \end{align}
        which correspond to equations $j = 0$ and $j = K$, respectively, of the system \eqref{eq:DICE:DiscreteCompat}.
    In summary we showed that a minimizer of the DICE loss satisfies \eqref{eq:DICE:DICEMinimizer:EqB} and thus also has to solve  the system of equations \eqref{eq:DICE:DiscreteCompat} corresponding to the time-discrete weak form of the continuity equation. 
    \end{proof}

\subsection{Existence and uniqueness of DICE minimizers}\label{appdx:ExistenceProof}
\begin{proof}[\Cref{prop:DICESolutionDiscreteInTime}]
\label{proof:DICESolutionDiscreteInTime} First, denote the Poincaré constant of \Cref{assumption:prelim:poincare} at time $t_j$ as $\lambda(t_j) = \lambda_j$ for $j = 0, \dots, K$. 
    We can then re-write the objective function in  \eqref{eq:DICEDiscreteDecoupled} as
    \begin{equation}\label{eq:DICE:Unique:LFun}
        \mathrm{L}_{\mathrm{DICE}}^{K}(\bar{\hat s}) = \sum_{j=0}^K  \int_{\mathcal X} \left( \frac{t_{j+1} - t_{j-1}}{4} |\nabla \hat s_j(x)|^2 - \frac{\rho(t_{j+1}, x) - \rho(t_{j-1}, x)}{2\rho(t_j, x)} \hat s_j (x) \right) \rho(t_j, x) \, \rd x\,.
    \end{equation}
    We obtain the bounds 
    \begin{align}
            \mathrm{L}_{\mathrm{DICE}}^{K}(\bar{\hat s}) &\geq \sum_{j=0}^K  \left( \frac{t_{j+1} - t_{j-1}}{4} \Vert \nabla \hat s_j \Vert_{L^2(\rho(t_j))}^2 - \left \Vert \frac{\rho(t_{j+1}, x) - \rho(t_{j-1}, x)}{2 \rho(t_j, x)} \right \Vert_{L^2(\rho(t_j))} \Vert \hat s_j \Vert_{L^2(\rho(t_j))} \right) \\
            &{\geq} \sum_{j=0}^K  \left( \lambda_j \frac{t_{j+1} - t_{j-1}}{4} \Vert \hat s_j \Vert_{L^2(\rho(t_j))}^2 - \left \Vert \frac{\rho(t_{j+1}, x) - \rho(t_{j-1}, x)}{2\rho(t_j, x)} \right \Vert_{L^2(\rho( t_j))} \Vert \hat s_j \Vert_{L^2(\rho(t_j))} \right) \\
            &\geq - \frac{1}{2} \sum_{j=0}^K \frac{2}{\lambda_j (t_{j+1} - t_{j-1})} \left \Vert \frac{\rho(t_{j+1}, x) - \rho(t_{j-1}, x)}{2\rho(t_j, x)} \right \Vert_{L^2(\rho( t_j))}^2,\label{eq:DICE:LowerBoundInProof}
        \end{align}
        where we used first the Cauchy-Schwarz inequality, then the Poincar\'e inequality, and then the fact that $\min_a \frac{ \lambda }{2} a^2 - ab = - \frac{1}{2\lambda} b^2$ for $a, b \in \mathbb{R}$ for $\lambda > 0$. Notice that when applying the Poincar\'e inequality, we assumed without loss of generality that $\mathbb{E}_{x \sim \rho(t_j)}[\hat{s}_j(x)] = 0$. We indeed can assume this without loss of generality because of \Cref{prop:InvariantF}, which also applies to the loss  \eqref{eq:DICE:Unique:LFun}. Thus, the loss value is unchanged:  $\mathrm{L}_{\mathrm{DICE}}^{K}(\bar{\hat s}) = \mathrm{L}_{\mathrm{DICE}}^{K}(\bar{\hat s}^0)$ for normalized $\bar{\hat{s}}^0$ that subtracts the constant $\mathbb{E}_{x \sim \rho(t_j)}[\hat{s}_j]$ from each component of $\bar{\hat{s}}$. 
        
        Given the quadratic nature of the objective function in \eqref{eq:DICEDiscreteDecoupled} and boundedness from below over $\mathcal{S}_0 \times \cdots \times \mathcal{S}_K$ as shown in  \eqref{eq:DICE:LowerBoundInProof}, existence and uniqueness of a minimizer now follows from standard arguments: We define the vector $\bar{l}(x) = [l_0, \dots, l_K]$ with components $j = 0, \dots, K$ given by 
        \begin{align}
            l_j(x) = \frac{1}{t_{j+1} - t_{j-1}} \frac{\rho(t_{j+1}, x) - \rho(t_{j-1}, x)}{ \rho(t_j, x)}.
        \end{align}
        Because \eqref{eq:DICErho_log_deriv_is_L2} holds, we obtain $\Vert \bar{l} \Vert_{L^2(\rho(t_0) \cdots \rho(t_K))}^2 < \infty$. We can then write the functional $\LD^K(\bar{\hat s})$ with \eqref{eq:DICE:Unique:LFun} as
        \begin{equation}\label{eq:DICE:LObjRewrite}
        \LD^K( \bar{\hat s}) = \frac{1}{2} \Vert \mathrm D \bar{\hat s} \Vert_{L^2(\rho(t_0) \cdots \rho(t_K))}^2 - \langle \bar{l}, \bar{\hat s}\rangle_{L^2(\rho(t_0) \cdots \rho(t_K))}\,.
    \end{equation}
    First, we again assume without loss of generality that $\bar{\hat{s}}$ has normalized components. Now notice that for any $\bar{\hat s} \in \mathcal{S}^0_0 \times \dots \times \mathcal{S}^0_K$ it holds: 
    \begin{align}
        \| \mathrm D \bar{\hat s} \|_{L^2(\rho(t_0) \cdots \rho(t_K))}^2
        &= \sum_{j=0}^K \frac{t_{j+1} - t_{j-1}}{2} \| \nabla \hat s_j \|_{L^2(\rho(t_j))}^2 \\
        &\geq \sum_{j=0}^K \frac{t_{j+1} - t_{j-1}}{2} \lambda_j \| \hat s_j \|_{L^2(\rho(t_j))}^2 \\
        &\geq \left(\min_{0 \leq j \leq K} \lambda_j \right) \sum_{j=0}^K \frac{t_{j+1} - t_{j-1}}{2} \|
        \hat s_j \|_{L^2(\rho(t_j))}^2 \\
        &= \left(\min_{0 \leq j \leq K} \lambda_j \right) \| \bar{\hat s} \|_{L^2(\rho(t_0) \cdots \rho(t_K)}^2.\label{appdx:ProofUniqueness:Existence:DSBoundHelper01A}
    \end{align}
Let now $( \bar{\hat s}_n)_n$ be a minimizing sequence with $\bar{\hat s}_n$ in $\mathcal{S}^0_0 \times \cdots \times \mathcal{S}^0_K$ such that  
    \begin{align}
        \lim_{n \to \infty} \LD^K(\bar{\hat s}_n) = \inf_{\bar{\bar{s}} \in \mathcal{S}_0 \times \dots \times \mathcal{S}_K} \LD^K(\bar{\bar{s}})\,,
    \end{align}
    with respect to the norm induced by the inner product $\langle \cdot, \cdot \rangle_{\mathcal{S}_0 \times \cdots \times \mathcal{S}_K}$.  Then, using \eqref{appdx:ProofUniqueness:Existence:DSBoundHelper01A} and the polarization identity $\frac{1}{2}|a|^2 + \frac{1}{2}|b|^2 = \frac{1}{4} |a - b| + \frac{1}{4} |a + b|$, we obtain 
        \begin{align}
            &\quad \, \max_{0 \leq j \leq K} \lambda_j^{-1}  \frac{1}{4} \Vert \bar{\hat s}_n - \bar{\hat s}_m \Vert_{L^2(\rho(t_0) \dots \rho(t_K))}^2 \nonumber \\
            &\leq \frac{1}{4} \Vert \mathrm D \bar{\hat s}_n - \mathrm D \bar{\hat s}_m \Vert_{L^2(\rho(t_0) \cdots \rho(t_K))}^2 \label{eq:DICE:InterABC01} \\
            &= \frac{1}{2} \Vert \mathrm D \bar{\hat s}_n \Vert_{L^2(\rho(t_0) \cdots \rho(t_K))}^2 + \frac{1}{2} \Vert \mathrm D \bar{\hat s}_m \Vert_{L^2(\rho(t_0) \cdots \rho(t_K))}^2 - \frac{1}{4} \Vert \mathrm D \bar{\hat s}_n + \mathrm D \bar{\hat s}_m \Vert_{L^2(\rho(t_0) \cdots \rho(t_K))}^2\,,
        \end{align}
        where \eqref{eq:DICE:InterABC01} is obtained via the Poincar\'e inequality. 
        We can add to this expression the term $0 = 2 \langle \bar{l}, \frac{1}{2} (\bar{\hat s}_n + \bar{\hat s}_n) \rangle 
        - \langle \bar{l}, \bar{\hat s}_n \rangle 
        - \langle \bar{l}, \bar{\hat s}_m \rangle$ (where $\langle \cdot, \cdot \rangle = \langle \cdot, \cdot \rangle_{L^2(\rho(t_0) \cdots \rho(t_K))}$) and collect terms to find
        \begin{align}
            \frac{1}{4} \Vert \mathrm D \bar{\hat s}_n - \mathrm D \bar{\hat s}_m \Vert_{L^2(\rho(t_0) \cdots \rho(t_K))}^2
            &= \LD^K( \bar{\hat{s}}_n) + \LD^K( \bar{\hat{s}}_m) - 2 \, \LD^{K}\left ( \frac{\bar{\hat s}_n + \bar{\hat s}_m}{2} \right ) \label{eq:DICE:InterABC02-1}\\ 
            &\leq \LD^K( \bar{\hat{s}}_n) + \LD^K( \bar{\hat{s}}_m) - 2 \, \inf_{\bar{\bar{s}} \in \mathcal{S}_0 \times \dots \times \mathcal{S}_K} \LD^K(\bar{\bar{s}} )\,,\label{eq:DICE:InterABC02}
        \end{align}
         where \eqref{eq:DICE:InterABC02} is obtained because $\LD^K$ can be written as in \eqref{eq:DICE:LObjRewrite}. 
        Given that $(\bar{\hat{ s}}_n)_n$ is a minimizing sequence, for every $\epsilon > 0$ there exists an $n^*$ such that 
        \begin{align}
        \LD^K( \bar{\hat{s}}_n) - \inf_{\bar{\bar s} \in \mathcal{S}_0 \times \dots \times \mathcal{S}_K} \LD^K\left (\bar{\bar s} \right ) \leq \epsilon\,,\qquad \forall n > n^*.
        \end{align}        
        Therefore, for every $n,m > n^*$ it holds that $\Vert \mathrm D\bar{\hat s}_n - \mathrm D\bar{\hat s}_m \Vert_{L^2(\rho(t_0) \cdots \rho(t_K))}^2 \leq 2 \epsilon$ and thus $\left(\min_{0 \leq j \leq K} \lambda_j \right)^{-1} \frac{1}{4} \Vert \bar{\hat s}_n - \bar{\hat s}_m \Vert_{L^2(\rho(t_0) \cdots \rho(t_K))}^2 \leq 2 \epsilon$. Hence, $(\bar{\hat{s}}_n)_n$ is a Cauchy sequence in $\langle \cdot, \cdot \rangle_{\mathcal{S}_0 \times \dots \times \mathcal{S}_K}$ and the infimum is attained by a minimizer, the limit of the Cauchy sequence, since $\mathcal{S}_0 \times \dots \times \mathcal{S}_K$ is a Hilbert space.
        
        The minimum is unique in $\mathcal{S}_0^0 \times \dots \times \mathcal{S}_K^0$ with respect to $\|\cdot\|_{\mathcal{S}_0 \times \dots \times \mathcal{S}_K}$ by strict convexity of $\LD^K$. Suppose there exist two solutions $\bar s^*, \bar s^{**} \in \mathcal{S}_0^0 \times \dots \times \mathcal{S}_K^0$, then
        \begin{align}
            \frac{1}{2} \left( \LD^K(\bar{s}^*) + \LD^K(\bar{s}^{**}) \right) = \frac{1}{2} \, 2 \min_{\bar{\bar s} \in \mathcal{S}_0^0 \times \dots \times \mathcal{S}_K^0} \LD^K\left (\bar{\bar s} \right )\,.
        \end{align}
        At the same time, the strict convexity of $\LD^K$ implies
        \begin{align}
            \LD^K \left( \frac{ \bar{s}^* + \bar{s}^{**}}{2} \right) < \frac{1}{2} \left( \LD^K(\bar{s}^*) + \LD^K(\bar{s}^{**}) \right) = \min_{\bar{\bar s} \in \mathcal{S}_0^0 \times \dots \times \mathcal{S}_K^0} \LD^K\left (\bar{\bar s} \right )
        \end{align}
        for all $\|\mathrm D \bar s^* - \mathrm D \bar s^{**}\|_{L^2(\rho(t_0) \dots \rho(t_K))} > 0$. Thus, by contradiction, because $\bar{s}^*, \bar{s}^**$ are normalized, it holds $0 = \|\mathrm D \bar s^* - \mathrm D \bar s^{**}\|_{L^2(\rho(t_0) \dots \rho(t_K))} = \|\bar{s}^* - \bar{s}^{**}\|_{\mathcal{S}_0 \times \dots \times \mathcal{S}_K} = 0$. 
    \end{proof}

    \begin{proof}[\Cref{corr:DICESolutionParametrizedInTime}]
        First, recall that $\LD(s) = \LD^K(\bar{s})$ for any $s \in \mathcal{S}$ with $\bar{s} = [s(t_0, \cdot), \dots, s(t_K, \cdot)]$. Let us now take the DICE minimizer $\hat s^* \in \mathcal{S}$ and evaluate it at the time points $t_0, \dots, t_K$ to obtain the vector of functions $\bar{v}^{*} = [\hat s^*(t_0, \cdot), \dots, \hat s^*(t_K, \cdot)]$, which is an element of the Cartesian product $\mathcal{S}_0 \times \dots \times \mathcal{S}_K$ per definition of $\mathcal{S}$. If $\LD^K(\bar{v}^*) < \LD^K(\bar{\hat{s}}^*)$, then $\bar{\hat{s}}^*$ cannot be a minimizer of $\LD^K$, which violates the assumption that $\bar{\hat{s}}^*$ is a minimizer of $\LD^K$. If $\LD^K(\bar{v}^*) > \LD^K(\bar{\hat{s}}^*)$, then $\hat{s}^*$ cannot be a minimizer of $\LD$ because stitching together the components of $\bar{\hat{s}}^*$ into a function over $\mathcal{S}$ would provide a lower value of $\LD$ than $\hat{s}^*$, which violates the assumption that $\hat{s}^*$ is a minimizer of $\LD$. Thus, we have $\LD^K(\bar{\hat{s}}^*) = \LD^K(\bar{v}^{*})$. This means that $\bar{v}^*$ is a minimizer of $\LD^K$ and thus the components of $\bar{v}^* = [\hat{s}^*(t_0, \cdot), \dots, \hat{s}^*(t_K, \cdot)]$ have to agree with the components of $\bar{\hat{s}}^*$ in the sense of \eqref{eq:DICEDiscSolEqual} due to the uniqueness of $\LD^K$ minimizers shown in \Cref{prop:DICESolutionDiscreteInTime}. 
    \end{proof}

    \begin{proof}[\Cref{corr:DICELowerBound}]
        The value of the loss $\LD(s)$ at any $s \in \mathcal{S}$ only depends on the vector of functions $\bar{s} = [s(t_0, \cdot), \dots, s(t_K, \cdot)] \in \mathcal{S}_0 \times \dots \times \mathcal{S}_K$ at the discrete time points $t_0, \dots, t_K$ and thus coincides with $\LD^K$. It was shown in \eqref{eq:DICE:LowerBoundInProof} in the proof of \Cref{prop:DICESolutionDiscreteInTime} that $\LD^K$ is lower bounded over $\mathcal{S}^0_0 \times \dots \times \mathcal{S}^0_K$. For any $\bar s \in \mathcal{S}_0 \times \dots \times \mathcal{S}_K$, we can construct $\bar s^0 = [s(t_0, \cdot) - \mathbb E_{\rho(t_0)}[s(t_0, \cdot)] , \dots, s(t_K, \cdot) - \mathbb E_{\rho(t_K)}[s(t_K, \cdot)] ] \in \mathcal{S}^0_0 \times \dots \times \mathcal{S}^0_K$ and we have $\LD^K(\bar s) = \LD^K(\bar s^0)$ by \Cref{prop:InvariantF} and $\LD^K(\bar s^0)$ is bounded below by \Cref{prop:DICESolutionDiscreteInTime}.
    \end{proof}

    \subsection{Infinite data limit}

    \begin{proof}[\Cref{prop:dice:bounds_on_rho}]
    Denote by
    $[ \nabla \cdot u ]^-$ the negative part of $\nabla \cdot u$, defined as 
    \begin{align}
        [ \nabla \cdot u ]^-(t, x) = \begin{cases}
            \nabla \cdot u(t, x)\,, & \text{if } \nabla \cdot u(t, x) < 0\,, \\
            0\,, & \text{otherwise.}
        \end{cases}
    \end{align}
    Then, it holds 
    \begin{align}
        - \nabla \cdot u(t, x) \leq |[ \nabla \cdot u ]^-| (t, x)\leq \max_{(t, y) \in [0, T] \times \mathcal{X}} | [ \nabla \cdot u ]^- | (t, y)=: \overline{c} \quad \forall x \in \mathcal X
    \end{align}
    and analogously
    \begin{align}
        -\nabla \cdot u(t, x) \geq - [ \nabla \cdot u ]^+ (t, x) \geq - \max_{(t, y) \in [0, T] \times \mathcal{X}} [ \nabla \cdot u ]^+(t, y) =: - \underline{c} \quad \forall x \in \mathcal X.
    \end{align}
    Then,
    \begin{align}
        \rho(t, x) 
        &= \rho_0(\phi_t^{-1}(x)) \exp\left( - \int_0^t (\nabla \cdot u)(\tau, \phi_\tau(\phi_t^{-1}(x))) \, d\tau \right) \\
        &\leq \rho_0(\phi_t^{-1}(x)) \exp\left( - \int_0^t \overline c \, d\tau \right) \\
        &\leq \overline{\rho_0} \exp\left( - \overline c t \right).
    \end{align}
    The lower bound follows analogously.
    \end{proof}

    \begin{proof}[\Cref{prop:DICEErrorBoundAtDataTimepoints}]
    A consequence of \Cref{asm:Smooth} and $\hat{s}^*$ satisfying \eqref{eq:DICE:Smooth:ContEqsHat} is
    \begin{align}
        \int_{\mathcal X} \varphi(x) (\partial_t - \hat \delta_{t_j}) \rho(t_j, x) \, \rd x &= \int_{\mathcal X} \nabla \varphi(x) \cdot \left(  \nabla s^*(t_j, x) - \nabla \hat s^*(t_j, x) \right) \rho(t_j, x) \, \rd x
    \end{align}
    for all $\varphi \in H^1_0(\mathrm dx)$ and $j = 0, \dots, K$. We used the fact that any function in $H^1_0(\mathrm dx)$ is also in $H^1_0(\rho(t_j))$ since $\rho$ is bounded above and below.
    The case $\varphi_j = s^*(t_j, \cdot) - \hat s^*(t_j, \cdot)$ gives
    \begin{align}
        \| \nabla s^*(t_j, \cdot) - \nabla \hat s^*(t_j, \cdot) \|_{L^2(\rho(t_j))}^2 
        &= \int_{\mathcal{X}} ( s^*(t_j, x) - \hat s^*(t_j, x)) (\partial_t - \hat \delta_{t_j}) \rho(t_j, x) \, \rd x \\
        &\leq \| s^*(t_j, x) - \hat s^*(t_j, x) \|_{L^2(\rd x)}  \| (\partial_t - \hat \delta_{t_j}) \rho(t_j, x) \|_{L^2(\rd x)} \\
        &\leq \lambda_{\mathcal X}^{-1/2} \| \nabla s^*(t_j, x) - \nabla \hat s^*(t_j, x) \|_{L^2(\rd x)}  \| (\partial_t - \hat \delta_{t_j}) \rho(t_j, x) \|_{L^2(\rd x)} \, ,
    \end{align}
   for $j = 0, \dots, K$. In the last step, we used the Poincaré inequality of the domain (with constant $\lambda_{\mathcal X}$). Then, with
    \begin{align}
        \| \nabla s^*(t_j, \cdot) - \nabla \hat s^*(t_j, \cdot) \|_{L^2(\rho(t))} \geq \left( \min_{x \in \mathcal{X}} \rho(t_j, x) \right) \| \nabla s^*(t_j, \cdot) - \nabla \hat s^*(t_j, \cdot) \|_{L^2(\rd x)},
    \end{align}
    we find with \Cref{prop:dice:bounds_on_rho} that 
    \begin{align}
        \| \nabla s^*(t_j, \cdot) - \nabla \hat s^*(t_j, \cdot) \|_{L^2(\rho(t_j))} \leq \lambda^{-1/2}_{\mathcal X} \underline{\rho}_0^{-1} \exp\left( \underline{c} t_j \right) \| (\partial_t - \hat \delta_{t_j}) \rho(t_j, x) \|_{L^2(\rd x)}
    \end{align}
    holds. 

    \end{proof}

\begin{proof}[\Cref{prop:DICE:BoundForAllT}]
    In this proof, we will abbreviate $s_t^* = s^*(\cdot, t), \hat{s}_j^* = \hat{s}^*(\cdot, t_j)$, $\hat{s}_t = \hat{s}(t, \cdot)$ as well as $\rho_j = \rho(t_j)$ and $\rho_t = \rho(t)$.
    Recall that $s^*_t$ and $\hat s^*_j$ satisfy
    \begin{align}
    \label{eq:dice:proof_between_tj_1}
        \partial_t \rho_t &= - \nabla \cdot (\rho_t \nabla s_t^*), \\
        \hat \delta_{t_j} \rho_j &= - \nabla \cdot (\rho_j \nabla \hat s_j^*)
    \end{align}
    in $H^1_0(\mathrm dx)$ with homogeneous Neumann boundary conditions as per \Cref{asm:Smooth} and \eqref{eq:DICE:Smooth:ContEqsHat}, respectively. We write the equations formally in strong form for the sake of readability. 
    
    We can write, by adding zero terms,
    \begin{align}
    \label{eq:dice:proof_between_tj_2}
        - \hat \delta_{t_j} \rho_j &= - \nabla \cdot ((\rho_t - \rho_j) \nabla \hat s_j^*) - \nabla \cdot( \rho_t (\nabla \hat s_t - \nabla \hat s_j^*)) + \nabla \cdot( \rho_t \nabla \hat s_t ).
    \end{align}
    Adding \eqref{eq:dice:proof_between_tj_1} and \eqref{eq:dice:proof_between_tj_2}, together with the definition of $\hat s_t$,
    \begin{align}
        \hat{s}_t\bigg |_{t \in [t_j, t_{j+1}]} = \frac{t_{j+1} - t}{t_{j+1} - t_j} \hat{s}^*_j + \frac{t - t_j}{t_{j+1} - t_j} \hat{s}^*_{j + 1}\,, \quad j = 0, \dots, K-1\,,
    \end{align}
    implying that
    \begin{align}
        \nabla \hat s_t - \nabla \hat s_j^* = \frac{t - t_j}{t_{j+1} - t_j} (\nabla \hat s_{j+1}^* - \nabla \hat s_j^*)\,,
    \end{align}
    gives us
    \begin{align}
    \label{eq:dice:proof_between_tj_3}
        \partial_t \rho_t - \hat \delta_{t_j} \rho_j & = - \nabla \cdot ((\rho_t - \rho_j) \nabla \hat s_j^*) - \frac{t - t_j}{t_{j+1} - t_j} \nabla \cdot (\rho_t (\nabla \hat s_{j+1}^* - \nabla \hat s_j^*)) + \nabla \cdot (\rho_t (\nabla \hat s_t - \nabla s_t^*)).
    \end{align}
    We can multiply this expression with $\hat s_t - s_t^*$ and integrate over $\mathcal X$ (i.e., test with this function) to obtain
    \begin{align}
    \label{eq:dice:proof_between_tj_4}
        \| \nabla \hat s_t - \nabla s_t^* \|^2_{L^2(\rho_t)} 
        &= \underbrace{- \int_{\mathcal X} (\hat s_t - s_t^*) (\partial_t \rho_t - \hat \delta_{t_j} \rho_j) \, \rd x}_{=: \, (i)} \nonumber \\
        &\quad \, \underbrace{+ \int_{\mathcal X} (\nabla \hat s_t - \nabla s_t^*) \cdot \nabla \hat s_j^* (\rho_t - \rho_j) \, \rd x}_{=: \, (ii)} \nonumber \\
        &\quad \,  \underbrace{+ \frac{t - t_j}{t_{j+1} - t_j} \int_{\mathcal X}  (\nabla \hat s_t - \nabla s_t^*) \cdot (\nabla \hat s_{j+1}^* - \nabla \hat s_j^*) \rho_t \, \rd x}_{=: \, (iii)}\,.
    \end{align}
    We now look at these terms one by one. We already see that $(i)$ depends on the accuracy of the finite-difference approximation, $(ii)$ depends on the change of $\rho$ in time, and $(iii)$ on the change of $\hat s$ in time.

    Let us consider term $(i)$ first: 
    We can expand $\rho_t$ about $t_{j + 1}$ to obtain  
    \begin{align}
    \rho_{j+1} = \rho_t + (t_{j+1} - t) \partial_t \rho_t + \frac{1}{2}(t_{j+1} - t)^2 \partial_t^2 \rho_{\tau} \end{align}
    with a Taylor remainder theorem for a $\tau \in [t, t_{j+1}]$ by using that $\rho$ is twice continuously differentiable in $t$.  We use the analogous expansion of $\rho_t$ about $t_{j - 1}$ to obtain $\rho_{j-1} = \rho_t - (t - t_{j-1}) \partial_t \rho_t + \frac{1}{2}(t - t_{j-1})^2 \partial_t^2 \rho_{\tau'} $ for a $\tau' \in [t_{j-1}, t]$. Substituting the two expansions into the definition of $\hat \delta_{t_j} \rho_j$ given in \eqref{eq:DICE:Smooth:DefDeltaRho} and simplifying gives
    \begin{align}
        \hat \delta_{t_j} \rho_j = \partial_t \rho_t + \frac{1}{2} \frac{(t_{j+1} - t)^2}{t_{j+1} - t_{j-1}} \partial_t^2 \rho_\tau - \frac{1}{2} \frac{(t - t_{j-1})^2}{t_{j+1} - t_{j-1}} \partial_t^2 \rho_{\tau'}\,, \qquad t \in [t_{j - 1}, t_{j + 1}]\,. 
    \end{align}
    Because $t \in [t_{j-1}, t_{j+1}]$, we have $\frac{(t_{j+1} - t)^2}{t_{j+1} - t_{j-1}} \leq t_{j+1} - t_j$ and $\frac{(t - t_{j-1})^2}{t_{j+1} - t_{j-1}} \leq t_{j} - t_{j-1}$, and thus
    \begin{align}
        (i) &\leq \| \hat s_t - s_t^* \|_{L^2(\rd x)} \| \partial_t \rho_t - \hat \delta_{t_j} \rho_j \|_{L^2(\rd x)} \\
        &\leq \| \hat s_t - s_t^* \|_{L^2(\rd x)} \frac{1}{2} \| (t_{j+1} - t_j) | \partial^2_t \rho_\tau | + (t_{j} - t_{j-1}) |\partial^2_t \rho_{\tau'}| \|_{L^2(\rd x)} \\
        &\leq \| \hat s_t - s_t^* \|_{L^2(\rd x)} \sup_{\tau \in [t_{j-1}, t_{j+1}]} \| \partial^2_t \rho_\tau \|_{L^2(\rd x)} \, \max \{ t_{j+1} - t_j, t_{j} - t_{j-1}\}\,.
    \end{align}
    Lastly, we can use the Poincar\'e inequality (which holds because of \Cref{asm:Smooth}) and boundedness of $\rho$ to get 
    \begin{align}
    \label{eq:DICE:Smooth:InterProof_bound_on_grads}
        \| \hat s_t - s_t^* \|_{L^2(\rd x)} \leq \lambda^{-1/2}_{\mathcal X} \| \nabla \hat s_t - \nabla s_t^* \|_{L^2(\rd x)} \leq \lambda^{-1/2}_{\mathcal X} \overline{\rho}_0 \exp\left( \overline{c} t \right) \| \nabla \hat s_t - \nabla s_t^* \|_{L^2(\rho_t)}
    \end{align}
    and conclude
    \begin{align}
    \label{eq:DICE:Smooth:InterProof_bound_i}
        (i) &\leq C_{(i)} \max \{ t_{j+1} - t_j, t_{j} - t_{j-1}\} \| \nabla \hat s_t - \nabla s_t^* \|_{L^2(\rho_t)}\,,
    \end{align}
    where $C_{(i)} = \lambda^{-1/2}_{\mathcal X} \overline{\rho}_0 \exp\left( \overline{c} t \right) \max_{\tau \in [t_{j-1}, t_{j+1}]} \| \partial^2_t \rho_\tau \|_{L^2(\rd x)}$.
    
    Let us now consider term $(ii)$ of \eqref{eq:dice:proof_between_tj_4}: 
    First, with the Cauchy-Schwarz inequality, we obtain
    \begin{align}
        (ii) &\leq \| \nabla \hat s_t - \nabla s_t^* \|_{L^2(\rd x)} \left \| \nabla \hat s_j^* (\rho_j - \rho_t) \right \|_{L^2(\rd x)} \\
        &\leq \| \nabla \hat s_t - \nabla s_t^* \|_{L^2(\rd x)} \left \| \nabla \hat s_j^* \right \|_{L^2(\rd x)} \max_{x \in \mathcal X} |\rho(t_j, x) - \rho(t, x)|\,.
    \end{align}
    The steps to \eqref{eq:DICE:Smooth:InterProof_bound_on_grads} can be repeated for the equation $\hat \delta_{t_j} \rho_j = - \nabla \cdot (\rho_j \nabla \hat s^*_j)$ to find
    \begin{align}
        \left \| \nabla \hat s_j^* \right \|_{L^2(\rd x)}  \leq \lambda^{-1/2}_{\mathcal X} \overline{\rho}_0 \exp\left( \overline{c} t_j \right) \| \hat \delta_{t_j} \rho_j \|_{L^2(\rd x)} \, .
    \end{align}
    The definition of $\hat \delta_{t_j}$ and
    \begin{align}
        | \rho(t_{j+1}) - \rho(t_{j-1}) | \leq |t_{j+1} - t_{j-1}| \max_{\tau \in [t_{j+1}, t_{j-1}]} | \partial_t \rho(\tau) |
    \end{align}
    then imply that
    \begin{align}
        \left \| \nabla \hat s_j^* \right \|_{L^2(\rd x)}  \leq \lambda^{-1/2}_{\mathcal X} \overline{\rho}_0 \exp\left( \overline{c} t_j \right) \| \hat \delta_{t_j} \rho_j \|_{L^2(\rd x)} \leq \lambda^{-1/2}_{\mathcal X} \overline{\rho}_0 \exp\left( \overline{c} t_j \right) \max_{\tau \in [t_{j-1}, t_{j+1}]} \| \partial_t \rho_\tau \|_{L^2(\rd x)} \, .
    \end{align}
    Second, by Taylor's theorem and because $\rho$ is assumed to be twice differentiable in $t$, there exists for all $x$ a $\tau \in [t_j, t]$ such that $|\rho(t_j, x) - \rho(t, x)| = |t - t_j| \partial_t \rho(\tau, x)$ holds, which means that we obtain the bound
    \begin{align}
    \label{eq:DICE:Smooth:InterProof_bound_ii}
        (ii) &\leq C_{(ii)} |t - t_j| \| \nabla \hat s_t - \nabla s_t^* \|_{L^2(\rho_t)}\,,
    \end{align}
    where the constant is
    \begin{align}
    C_{(ii)} = \lambda^{-1/2}_{\mathcal X} \overline{\rho}_0^2 \exp \left( \overline{c} (t + t_j) \right) \left( \max_{x \in \mathcal X} \max_{\tau \in [t_{j}, t_{j+1}]} |\partial_t \rho(x, \tau)| \right) \max_{\tau \in [t_j, t_{j+1}]} \| \partial_t \rho_\tau \|_{L^2(\rd x)}.
    \end{align}
    The constant $C_{(ii)}$ is finite because $\rho$ is continuously differentiable in time.

    Let us now consider term $(iii)$ of \eqref{eq:dice:proof_between_tj_4}: 
    From the Cauchy-Schwarz inequality and $t - t_j \leq t_{j+1} - t_j$ for $t \in [t_j, t_{j+1}]$, we obtain 
    \begin{align}\label{eq:DICE:Smooth:InterProofABC002}
        (iii) \leq \| \nabla \hat s_t - \nabla s_t^* \|_{L^2(\rho_t)} \left \| \nabla \hat s_{j+1}^* - \nabla \hat s_j^* \right \|_{L^2(\rho_t)}.
    \end{align}
    To estimate the second factor in \eqref{eq:DICE:Smooth:InterProofABC002}, we can repeat the analogous procedure as in the derivation of \eqref{eq:dice:proof_between_tj_4}:  use 
    \begin{align}
        \hat \delta_{t_{j+1}} \rho_{j+1} = - \nabla \cdot (\rho_{j+1} \nabla \hat s_{j+1}^*)
    \end{align}
    to find
    \begin{align}
    \label{eq:dice:proof_between_tj_5}
        \hat \delta_{t_{j+1}} \rho_{j+1} - \hat \delta_{t_j} \rho_{j} 
        &= - \nabla \cdot (\rho_{j+1} \nabla \hat s_{j+1}^* - \rho_{j} \nabla \hat s_{j}^*) \\
        &= - \nabla \cdot ((\rho_{j+1} - \rho_t) \nabla \hat s_{j+1}^* + (\rho_t - \rho_{j}) \nabla \hat s_{j}^*) - \nabla \cdot (\rho_t (\nabla \hat s_{j+1}^* - \nabla \hat s_{j}^*)).
    \end{align}
    When $j > K - 2$, we can use $\hat \delta_{t_{j-1}} \rho_{j-1}$ in place of $\hat \delta_{t_{j+1}} \rho_{j+1}$ here without changing the argument. We test this equation with $\hat s_{j+1}^* - \hat s_j^*$ and obtain
    \begin{align}
        \| \nabla \hat s_{j+1}^* - \nabla \hat s_j^* \|^2_{L^2(\rho_t)} 
        &= \underbrace{\int_{\mathcal X} (\nabla \hat s_{j+1}^* - \nabla \hat s_j^*) \cdot \nabla \hat s_{j+1}^* (\rho_t - \rho_{j+1}) \, \rd x}_{=: \, (iv)} \nonumber \\
        &\quad \, \underbrace{+ \int_{\mathcal X} (\nabla \hat s_{j+1}^* - \nabla \hat s_j^*) \cdot \nabla \hat s_j^* (\rho_j - \rho_t) \, \rd x}_{=: \, (v)} \nonumber \\
        &\quad \, \underbrace{+ \int_{\mathcal X} (\hat s_{j+1}^* - \hat s_j^*) (\hat \delta_{t_{j+1}} \rho_{j+1} - \hat \delta_{t_j} \rho_{j}) \, \rd x}_{=: \, (vi)}\,.
    \end{align}
    For the terms $(iv)$ and $(v)$ we obtain the bounds $(iv) \leq C_{(iv)} |t_{j+1} - t| \| \nabla \hat s_{j+1}^* - \nabla \hat s_j^* \|_{L^2(\rho_t)}$ and $(v) \leq C_{(v)} |t - t_j| \| \nabla \hat s_{j+1}^* - \nabla \hat s_j^* \|_{L^2(\rho_t)}$ using the analogous arguments as used to bound $(ii)$, respectively with $C_{(v)} = C_{(ii)}$ and
    \begin{align}
        C_{(iv)} &= \overline{\rho}_0^2 \exp \left( \overline{c} (t + t_{j+1}) \right) \left( \max_{x \in \mathcal X} \max_{\tau \in [t_{j}, t_{j+1}]} |\partial_t \rho(x, \tau)| \right) \max_{\tau \in [t_j, t_{j+2}]} \| \partial_t \rho_\tau \|_{L^2(\rd x)}.
    \end{align}
    The term $(vi)$ can be bounded as
    \begin{align}\label{eq:ProofPropContTime:Aux3441}
        (vi) \leq \| \hat s^*_{j+1} - \hat s^*_j \|_{L^2(\rd x)} \left( \| \partial_t \rho_t - \hat \delta_{t_{j+1}} \rho_{j+1} \|_{L^2(\rd x)} + \| \partial_t \rho_t - \hat \delta_{t_j} \rho_j \|_{L^2(\rd x)}\right)\,.
    \end{align}
    The second term in \eqref{eq:ProofPropContTime:Aux3441} is analogous to the estimate $(i)$. The first term is also analogous but for the interval $[t_j, t_{j+2}]$ instead of $[t_{j-1}, t_{j + 1}]$. Hence,
    \begin{align}
        \| \nabla \hat s^*_{j+1} - \nabla \hat s_j^* \|^2_{L^2(\rho_t)} 
        &\leq C_{(iv)} |t_{j+1} - t| \| \nabla \hat s^*_{j+1} - \nabla \hat s^*_j \|_{L^2(\rho_t)} \nonumber \\
        &\quad \, + C_{(v)} |t - t_j| \| \nabla \hat s^*_{j+1} - \nabla \hat s^*_j \|_{L^2(\rho_t)} \nonumber \\
        &\quad \, + C_{(vi)} \max \{ t_{j+2} - t_{j+1}, t_{j+1} - t_j, t_{j} - t_{j-1}\} \| \nabla \hat s^*_{j+1} - \nabla \hat s_j^* \|_{L^2(\rho_t)}\,,
    \end{align}
    where 
    \begin{align}
        C_{(vi)} = C_{(i)} + \lambda^{-1/2}_{\mathcal X} \overline{\rho_0} \exp\left( \overline{c} t \right) \max_{\tau \in [t_{j}, t_{j+2}]} \| \partial^2_t \rho_\tau \|_{L^2(\rd x)}.
    \end{align}
    This results in the estimate
    \begin{align}
        \| \nabla \hat s^*_{j+1} - \nabla \hat s_j^* \|_{L^2(\rho_t)} &\leq (C_{(iv)} + C_{(v)} + C_{(vi)} ) \max \{ t_{j+2} - t_{j+1}, t_{j+1} - t_j, t_{j} - t_{j-1}\}
    \end{align}
    and hence
    \begin{align}
    \label{eq:DICE:Smooth:InterProof_bound_iii}
        (iii) \leq \| \nabla \hat s_t - \nabla s_t^* \|_{L^2(\rho_t)} (C_{(iv)} + C_{(v)} + C_{(vi)} ) \max \{ t_{j+2} - t_{j+1}, t_{j+1} - t_j, t_{j} - t_{j-1}\}
    \end{align}
    Now taking the bounds for all three terms $(i) - (iii)$ together (i.e. Equations \eqref{eq:DICE:Smooth:InterProof_bound_i}, \eqref{eq:DICE:Smooth:InterProof_bound_ii}, \eqref{eq:DICE:Smooth:InterProof_bound_iii}) and collecting all terms, we thus find (using $t-t_j < t_{j+1} - t_j$):
    \begin{align}
        \| \nabla \hat s_t - \nabla s^*_t \|_{L^2(\rho_t)} \leq \left( C_{(i)} + \ldots + C_{(vi)} \right) \max \{ t_{j+2} - t_{j+1}, t_{j+1} - t_j, t_{j} - t_{j-1}\}
    \end{align}
    for any $t \in [t_j, t_{j+1}]$. Taking the maximum over $j$ proves the claim. Note that
    \begin{align}
        C_{(i)}, \frac 1 2 C_{(vi)} \leq \lambda^{-1/2}_{\mathcal X} \overline{\rho}_0 \exp\left( \overline{c} T \right) \max_{\tau \in [0, T]} \| \partial^2_t \rho_\tau \|_{L^2(\rd x)}
    \end{align}
    and
    \begin{align}
    C_{(ii)}, C_{(iv)}, C_{(vi)} \leq \lambda^{-1/2}_{\mathcal X} \overline{\rho}_0^2 \exp \left( 2\overline{c} T \right) \left( \max_{\tau \in [0, T]} \|\partial_t \rho(\tau) \|_{L^\infty(\rd x)} \right) \max_{\tau \in [0, T]} \| \partial_t \rho_\tau \|_{L^2(\rd x)} \, .
    \end{align}

\end{proof}

\begin{proof}[\Cref{corr:DICE:BoundForAllT_NNDiscretization}]
    The only part of the proof of \Cref{prop:DICE:BoundForAllT} that needs to be modified in this case is the bound for $(iii)$ in \eqref{eq:dice:proof_between_tj_4}. It can be replaced by
    \begin{align}
        \int_{\mathcal X}  (\nabla s^*_t - \nabla \hat s_t) \cdot (\nabla \hat s_{t} - \nabla s^*_j) \rho_t \, \rd x 
        &\leq \| \nabla \hat s_t - \nabla s^*_t \|_{L^2(\rho_t)} \| \nabla \hat s_t - \nabla \hat s^*_j \|_{L^2(\rho_t)} \\
        &\leq \| \nabla \hat s_t - \nabla s^*_t \|_{L^2(\rho_t)} L^t_{\hat s} |t - t_j|.
    \end{align}
    The remainder of the proof is unchanged.
\end{proof}

\subsection{Inference error bound}

\begin{proof}[\Cref{prop:bound_on_inference_error_w2}]
We start by recalling the following result from \citet[Theorem 8.13]{villani_topics_2016}: 
        Let $(\epsilon, \epsilon) \ni t \mapsto \rho(t)$ be an absolutely continuous curve through $\mathcal P(\xdomain)$ and $\sigma \in \mathcal P(\xdomain)$ be a fixed absolutely continuous measure. When $\partial_t \rho + \nabla \cdot (\rho v) = 0$ for a bounded $v$ of class $C^1$ in both $t$ and $x$, then
        \begin{align}
            \frac{\rd}{\rd t} \frac{1}{2} W_2(\rho(t), \sigma)^2 = \int_\xdomain (x - T_{\rho(t) \rightarrow \sigma}(x)) \cdot v(x,t) \rho(t, x) \, \rd x,
        \end{align}
        where $T_{\rho(t) \rightarrow \sigma}$ denotes the optimal transport map from $\rho(t)$ to $\sigma$.

Let us now return to the proof of \Cref{prop:bound_on_inference_error_w2}. 
        The claim follows from \citet[Theorem 8.13]{villani_topics_2016} stated above together with the fact that
        $\Vert\mathrm{id} - T_{\rho(t) \rightarrow \hat \rho(t)} \Vert_{L^2(\rho(t))} = W_2(\rho(t), \hat \rho(t))$, where $T_{\rho(t) \rightarrow \hat \rho(t)}$ denotes the optimal transport map from $\rho(t)$ to $\hat \rho(t)$ and $\operatorname{id}$ the identity map.
        We obtain 
        \begin{align}
            \frac{\rd}{\rd t} \frac{1}{2} W_2(\rho(t), \hat \rho(t))^2
            &= \int_\xdomain (\mathrm{id} - T_{\rho(t) \rightarrow \hat \rho(t)}) \cdot \nabla s^*(t, \cdot) \rho(t) \, \rd x +\notag\\
            & \qquad \qquad\qquad \int_\xdomain (\mathrm{id} - T_{\hat \rho(t) \rightarrow \rho(t)}) \cdot \nabla \hat s^*(t, \cdot) \hat \rho(t) \, \rd x \\
            &= \int_\xdomain (\mathrm{id} - T_{\rho(t) \rightarrow \hat \rho(t)}) \cdot \nabla s^*(t, \cdot) \rho(t) \, \rd x\notag\\
            & \qquad\qquad\qquad\qquad + \int_\xdomain (T_{\rho(t) \rightarrow \hat \rho(t)} - \mathrm{id}) \cdot \nabla \hat s^*(t, \cdot) \circ T_{\rho(t) \rightarrow \hat \rho(t)} \rho(t) \, \rd x \\
            &= \int_\xdomain (\mathrm{id} - T_{\rho(t) \rightarrow \hat \rho(t)}) \cdot \left( \nabla s^*(t, \cdot) - \nabla \hat s^*(t, \cdot) \circ T_{\rho(t) \rightarrow \hat \rho(t)} \right) \rho(t) \, \rd x \\
            &\leq \Vert\mathrm{id} - T_{\rho(t) \rightarrow \hat \rho(t)} \Vert_{L^2(\rho(t))} \Vert \nabla s^*(t, \cdot) - \nabla \hat s^*(t, \cdot) \circ T_{\rho(t) \rightarrow \hat \rho(t)} \Vert_{L^2(\rho(t))} \\
            &= W_2(\rho(t), \hat \rho(t)) \, \Vert \nabla s^*(t, \cdot) - \nabla \hat s^*(t, \cdot) \circ T_{\rho(t) \rightarrow \hat \rho(t)} \Vert_{L^2(\rho(t))} \\
            &\leq W_2(\rho(t), \hat \rho(t)) \left( \Vert \nabla s^*(t, \cdot) - \nabla \hat s^*(t, \cdot) \Vert_{L^2(\rho(t, \cdot))}\right.\notag\\
             & \left.\qquad\qquad\qquad + \Vert \nabla \hat s^*(t, \cdot) - \nabla \hat s^*(t, \cdot) \circ T_{\rho(t, \cdot) \rightarrow \hat \rho(t, \cdot)} \Vert_{L^2(\rho(t))} \right)\,.
        \end{align}
        By Taylor's theorem, there exists a point $y$ on the line connecting $x$ to $T_{\rho(t) \rightarrow \hat \rho(t)}(x)$ such that
        \begin{align}
            \nabla \hat s^*(t, \cdot) \circ T_{\rho(t) \rightarrow \hat \rho(t)}(x) = \nabla \hat s^*(t, x) + D^2 \hat s^*(t, y) (T_{\rho(t) \rightarrow \hat \rho(t)}(x) - x)
        \end{align}
        and therefore
        \begin{align}
            \Vert \nabla \hat s^*(t, \cdot) - \nabla \hat s^*(t, \cdot) \circ T_{\rho(t) \rightarrow \hat \rho(t)} \Vert_{L^2(\rho)} &\leq \sup_{x \in \mathcal X} \Vert D^2_x \hat s^*(t, x) \Vert_{\mathrm{op}} \| T_{\rho(t) \rightarrow \hat \rho(t)} - \mathrm{id} \|_{L^2(\rho(t))}\\
            & = L_{\hat s^*}^x(t) W_2(\rho(t), \hat \rho(t)).
        \end{align}
        Together, this implies that
        \begin{align}
            \frac{\rd}{\rd t} W_2(\rho(t), \hat \rho(t))
            &\leq \Vert \nabla s^*(t, \cdot) - \nabla \hat s^*(t, \cdot) \Vert_{L^2(\rho(t))} + L_{\hat s^*}^x(t) W_2(\rho(t), \hat \rho(t))
        \end{align}
        and thus 
        \begin{align}
            W_2(\rho(t), \hat \rho(t)) &\leq W_2(\rho(0), \hat \rho(0)) + \int_0^t \Vert \nabla s^*(\cdot, \tau) - \nabla \hat s^*(\cdot, \tau) \Vert_{L^2(\rho(\tau))} \, \rd \tau \nonumber \\
            &\quad \, + \int_0^t L_{\hat s^*}^x(\tau) W_2(\rho(\tau), \hat \rho(\tau)) \, \rd \tau.
        \end{align}
        Plugging in the bound of \Cref{prop:DICE:BoundForAllT} and an application of the Gr\"onwall inequality concludes the proof.
    \end{proof}

\section{Details about the numerical experiments}\label{appdx:DetailsNumExp}

\paragraph{Setup and parameters}We use the entropic DICE loss from \Cref{eq:Loss_DICE_entropic} for Experiments 4.2 ($\varepsilon = 10^{-2}$) and 5 ($\varepsilon = 1.25 \times 10^{-1}$).

\paragraph{Experiment 1: No dynamics}~\\

\begin{tabular}{l|l}
    Network Architecture & MLP with swish activation functions \\
    Network size (width $\times$ depth) & $32 \times 3$ \\
    Domain & $x \in \bR, t \in [0,1], \mu = 0$ \\
    Data size ($N_x \times N_t \times N_\mu$) & $10^4 \times 512 \times 1$ \\
    Batch size ($n_x \times n_t \times n_\mu)$ & $128 \times 128 \times 1$ \\
    Optimizer & ADAM, cosine schedule $5 \times 10^{-4} \to 10^{-6}$ \\
    & $2 \times 10^4$ iterations \\
\end{tabular}

\paragraph{Experiment 2: Known potential}~\\

\begin{tabular}{l|l}
    Network Architecture & MLP with swish activation functions \\
    Network size (width $\times$ depth) & $128 \times 7$ \\
    Domain & $x \in \bR^2, t \in [0,1], \mu = 0$ \\
    Data size ($N_x \times N_t \times N_\mu$) & $2048 \times 1000 \times 1$ \\
    Batch size ($n_x \times n_t \times n_\mu)$ & $256 \times 256 \times 1$ \\
    Optimizer & ADAM, cosine schedule $5 \times 10^{-4} \to 10^{-6}$ \\
    & $10^3$ iterations \\
\end{tabular}

\paragraph{Experiment 3: Random Waves}~\\

\begin{tabular}{l|l}
    Network Architecture & Standard ResNet architecture\\
    Residual block features & $16 \times 32 \times 32$ \\
    Domain & $x \in [0, 1]^{1024}, t \in [0,8]$ \\
    Data size ($N_x \times N_t \times N_\mu$) & $4096 \times 512 \times 1$ \\
    Batch size ($n_x \times n_t \times n_\mu)$ & $32 \times 64 \times 1$ \\
    Optimizer & ADAM, cosine schedule $5 \times 10^{-4} \to 10^{-6}$ \\
    & $2 \times 10^4$ iterations \\
    Fluid simulation parameters & {viscosity} $= 10^{-3}$, {max\_velocity} $= 7$, {resolution} $= 256$
\end{tabular}

\paragraph{Experiment 4.1: Vlasov-Poisson (6D)}~\\

\begin{tabular}{l|l}
    Network Architecture & MLP with swish activation functions \\
    Network size (width $\times$ depth) & $128 \times 7$ \\
    Domain & $x \in [0, 4\pi] \times [0,1] \times [0,1] \times \bR^3, t \in [0, 8.75], \mu \in [0.5, 1.5]$ \\
    Data size ($N_x \times N_t \times N_\mu$) & $25000 \times 175 \times 8$ \\
    Batch size ($n_x \times n_t \times n_\mu)$ & $512 \times 128 \times 1$ \\
    Optimizer & ADAM, cosine schedule $5 \times 10^{-4} \to 10^{-6}$ \\
    & $3.3 \times 10^4$ iterations
\end{tabular}

\paragraph{Experiment 4.2: Vlasov-Poisson (2D)}~\\

\begin{tabular}{l|l}
    Network Architecture & MLP with swish activation functions \\
    Network size (width $\times$ depth) & $128 \times 7$ \\
    Domain & $x \in [0, 50] \times \bR, t \in [0, 40], \mu \in [1.2, 2.0]$ \\
    Data size ($N_x \times N_t \times N_\mu$) & $25000 \times 175 \times 8$ \\
    Batch size ($n_x \times n_t \times n_\mu)$ & $512 \times 128 \times 1$ \\
    Optimizer & ADAM, cosine schedule $5 \times 10^{-4} \to 10^{-6}$ \\
    & $5 \times 10^4$ iterations
\end{tabular}

\paragraph{Experiment 5: Raleigh-Bernard convection}~\\

\begin{tabular}{l|l}
    Network Architecture & MLP with swish activation functions \\
    Network size (width $\times$ depth) & $128 \times 7$ \\
    Domain & $x \in \bR^9, t \in [0, 20], \mu = 0$ \\
    Data size ($N_x \times N_t \times N_\mu$) & $10000 \times 2000 \times 8$ \\
    Batch size ($n_x \times n_t \times n_\mu)$ & $512 \times 256 \times 1$ \\
    Optimizer & ADAM, cosine schedule $5 \times 10^{-4} \to 10^{-6}$ \\
    & $2 \times 10^4$ iterations
\end{tabular}

\paragraph{Vlasov-Poisson bump-on-tail and two-stream instability} \label{appdx:2DVlasov}
\Cref{fig:two_fullwidth_images-v2_1} shows analogous results to the ones in Section~\ref{sec:NumExp:Vlasov} but for a bump-on-tail and two-stream Vlasov-Poisson instability. The setup follows the one described by \cite{berman2024parametric}. 

\begin{figure}
  \centering
  \includegraphics[width=0.7\textwidth]{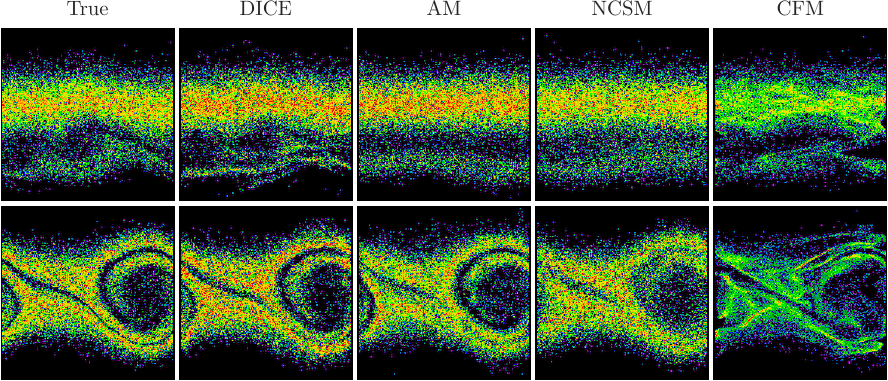}
  \vspace{1em}\includegraphics[width=0.7\textwidth]{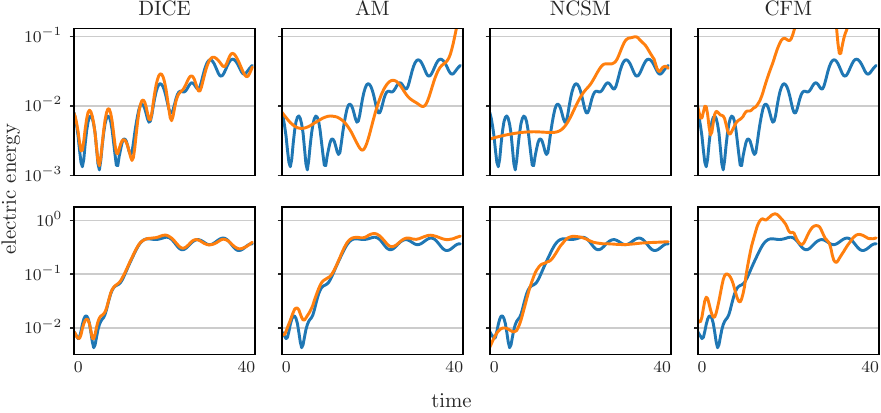}
  \caption{Vlasov-Poisson instabilities: Histograms (top), electric energy curves (bottom). Each plot: bump-on-tail instability (top) and two-stream instability (bottom).}
  \label{fig:two_fullwidth_images-v2_1}
\end{figure}

\end{document}